\documentclass[a4paper]{article}

\usepackage[utf8x]{inputenc}
\usepackage[T1]{fontenc}

\usepackage[a4paper,margin=1in]{geometry}

\usepackage[numbers]{natbib}
\usepackage[usenames,dvipsnames]{xcolor}

\definecolor{DarkRed}{rgb}{0.368,0.097,0.078}

\usepackage[colorlinks = true,
    linkcolor = blue,
    anchorcolor = blue,
    citecolor = blue,
    filecolor = blue,
    urlcolor = DarkRed,
    pagebackref]{hyperref}

\usepackage{amsthm} % better to load after ams math
\usepackage{mathtools,thmtools}

\usepackage{enumitem}
\usepackage{bm}
\usepackage{xspace}
\usepackage{nicefrac}
\usepackage[capitalise]{cleveref}
\usepackage{dsfont}

\declaretheoremstyle[
	    spaceabove=\topsep, 
	    spacebelow=\topsep, 
	    headfont=\normalfont\bfseries,
	    bodyfont=\normalfont\itshape,
	    notefont=\normalfont\bfseries,
	    notebraces={(}{)},
	    postheadspace=0.5em, 
	    headpunct={},
	    postfoothook=\noindent\ignorespaces
    ]{theorem}
\declaretheorem[style=theorem,numberwithin=section]{theorem}

\declaretheoremstyle[
	    spaceabove=\topsep, 
	    spacebelow=\topsep, 
	    headfont=\normalfont\bfseries,
	    bodyfont=\normalfont,
	    notefont=\normalfont\bfseries,
	    notebraces={(}{)},
	    postheadspace=0.5em, 
	    headpunct={},
	    postfoothook=\noindent\ignorespaces
    ]{definition}

\declaretheoremstyle[
        spaceabove=\topsep, 
        spacebelow=\topsep, 
        headfont=\normalfont\bfseries,
        bodyfont=\normalfont,
        notefont=\normalfont\bfseries,
        notebraces={}{},
        postheadspace=0.5em, 
        qed=$\blacksquare$, 
        headpunct={},
        postfoothook=\noindent\ignorespaces
    ]{proofstyle}
\declaretheorem[style=proofstyle,numbered=no,name=Proof]{proof}

\declaretheorem[style=theorem,sibling=theorem,name=Lemma]{lemma}
\declaretheorem[style=theorem,sibling=theorem,name=Corollary]{corollary}

\declaretheorem[style=theorem,sibling=theorem,name=Proposition]{proposition}

\declaretheorem[style=theorem,numbered=no,name=Theorem]{theorem*}
\declaretheorem[style=theorem,numbered=no,name=Lemma]{lemma*}
\declaretheorem[style=theorem,numbered=no,name=Corollary]{corollary*}
\declaretheorem[style=theorem,numbered=no,name=Proposition]{proposition*}
\declaretheorem[style=theorem,numbered=no,name=Claim]{claim*}
\declaretheorem[style=theorem,numbered=no,name=Fact]{fact*}
\declaretheorem[style=theorem,numbered=no,name=Observation]{observation*}
\declaretheorem[style=theorem,numbered=no,name=Conjecture]{conjecture*}

\declaretheorem[style=definition,sibling=theorem,name=Definition]{definition}
\declaretheorem[style=definition,sibling=theorem,name=Remark]{remark}

\declaretheorem[style=definition,numbered=no,name=Definition]{definition*}
\declaretheorem[style=definition,numbered=no,name=Remark]{remark*}
\declaretheorem[style=definition,numbered=no,name=Example]{example*}
\declaretheorem[style=definition,numbered=no,name=Question]{question*}

\DeclareMathAlphabet{\mathbfsf}{\encodingdefault}{\sfdefault}{bx}{n}

\DeclareMathOperator*{\argmin}{arg\!\min}

\DeclareMathOperator*{\supp}{supp}

\let\Pr\relax
\DeclareMathOperator{\Pr}{\mathbb{P}}

\newcommand{\lr}[1]{\mathopen{}\left(#1\right)}

\newcommand{\LR}[1]{\mathopen{}\Big(#1\Big)}

\newcommand{\norm}[1]{\|#1\|}

\newcommand{\lrset}[1]{\mathopen{}\left\{#1\right\}}

\usepackage[T1]{fontenc}
\usepackage{times}

\newcommand{\VC}{\operatorname{VC}}
\newcommand{\Risk}{\operatorname{Risk}}

\renewcommand{\phi}{\varphi}
\usepackage{stmaryrd}

\newcommand{\I}{\mathbb{I}}

\newcommand{\NN}{\mathbb{N}}
\newcommand{\N}{\mathbb{N}}

\newcommand{\paren}[1]{\left( #1 \right)}
\newcommand{\sqparen}[1]{\left[ #1 \right]}

\newcommand{\sett}[1]{\left\{ #1 \right\}}

\newcommand{\beq}{\begin{eqnarray*}}
\newcommand{\eeq}{\end{eqnarray*}}
\newcommand{\beqn}{\begin{eqnarray}}
\newcommand{\eeqn}{\end{eqnarray}}
\newcommand{\ben}{\begin{enumerate}}
\newcommand{\een}{\end{enumerate}}
\newcommand{\bit}{\begin{itemize}}
\newcommand{\eit}{\end{itemize}}

\newcommand{\calL}{\mathcal{L}}
\newcommand{\M}{\mathcal{M}}

\newcommand{\F}{\mathcal{F}}

\newcommand{\E}{\mathbb{E}}

\newcommand{\U}{\mathcal{U}}
\newcommand{\X}{\mathcal{X}}

\newcommand{\A}{\mathcal{A}}
\newcommand{\B}{\mathcal{B}}
\newcommand{\Y}{\mathcal{Y}}

\newcommand{\D}{\mathcal{D}}
\renewcommand{\H}{\mathcal{H}}
\renewcommand{\O}{\mathcal{O}}

\usepackage{scalerel}
\usepackage{stackengine}

\RequirePackage{mathtools}
\RequirePackage{amsmath}
\RequirePackage{stackrel}
\RequirePackage{dsfont}
\RequirePackage{lipsum}
\usepackage{xspace}

\usepackage[capitalise]{cleveref}

\usepackage{tablefootnote}
\usepackage{multirow}

\newcommand{\ERM}{\operatorname{ERM}}
\newcommand{\RERM}{\operatorname{RERM}}

\newcommand{\dualVC}{\operatorname{VC^*}}
\usepackage{algorithm}
\usepackage{algpseudocode}
\newcommand{\maj}{\operatorname{Majority}}

\newcommand{\VCU}{\operatorname{VC_{\U}}}
\newcommand{\RE}{\operatorname{RE}}
\newcommand{\AG}{\operatorname{AG}}
\newcommand{\risk}{\operatorname{R}}

\newcommand{\pac}{\operatorname{PAC}}
\newcommand{\grass}{\operatorname{GRASS}}

\usepackage{amsfonts,amsmath,amssymb}
\usepackage[rightcaption]{sidecap}

\title{A Characterization of Semi-Supervised Adversarially Robust PAC Learnability}

\author{
    Idan Attias \thanks{Department of Computer Science, Ben-Gurion University; \texttt{idanatti@post.bgu.ac.il.}} 
    \and  Steve Hanneke\thanks{Department of Computer Science, Purdue University; \texttt{steve.hanneke@gmail.com.}}
    \and Yishay Mansour \thanks{Blavatnik School of Computer Science, Tel Aviv University and Google Research;
  \texttt{mansour.yishay@gmail.com.}}
}

\begin{document}
\maketitle

\begin{abstract}%
We study the problem of learning an adversarially robust predictor to test time attacks in the \textit{semi-supervised} $\pac$ model.
We address the question of how many \textit{labeled} and \textit{unlabeled} examples are required to ensure learning.
We show that having enough unlabeled data (the size of a labeled sample that a fully-supervised method would require),
the labeled sample complexity can be arbitrarily smaller compared to previous works, and is sharply characterized by a \textit{different} complexity measure. We prove nearly matching upper and lower bounds on this sample complexity.
This shows that there is a significant benefit in semi-supervised robust learning even in the worst-case distribution-free model, and establishes
a gap between supervised and semi-supervised label
complexities which is known not to hold in standard non-robust $\pac$ learning.
\end{abstract}
\section{Introduction}
The problem of learning predictors that are immune to adversarial corruptions at inference time is central in modern machine learning. 
The phenomenon of fooling learning models by adding imperceptible perturbations to their input illustrates a basic vulnerability of learning-based models, and is named \textit{adversarial examples}.
We study the model of adversarially-robust $\pac$ learning, in a \textit{semi-supervised} setting. %

Adversarial robustness has been shown to 
significantly benefit from semi-supervised learning, mostly empirically, but also theoretically in some specific cases of distributions  \cite[e.g.,][]{carmon2019unlabeled,zhai2019adversarially,uesato2019labels,najafi2019robustness,alayrac2019labels,wei2020theoretical,levi2021domain}.
In this paper we ask the following natural question. To what extent can we benefit from \textit{unlabeled} data in the learning process of robust models in the general case? 
More specifically, what is the sample complexity in a distribution-free model?

Our semi-supervised model is formalized as follows.
Let $\H\subseteq \{0,1\}^\X$ be a hypothesis class.
We formalize the adversarial attack by a perturbation function $\U:\X\rightarrow 2^\X$, where $\U(x)$ is the set of possible perturbations (attacks) on $x$. In practice, we usually consider $\U(x)$ to be the $\ell_p$ ball centered at $x$. 
In this paper, we have no restriction on $\U$, besides $x\in\U(x)$.
The robust error of hypothesis $h$ on a pair $(x,y)$ is $\sup_{z\in \U(x)}\I\sqparen{h(z)\neq y}$.
The learner has access to both \textit{labeled} and \textit{unlabeled} examples drawn i.i.d. from unknown distribution $\D$, and the goal is to find $h\in\H$ with low robust error on a random point from $\D$. The sample complexity in semi-supervised learning has two parameters, the number of labeled examples and the number of unlabeled examples which suffice to ensure learning. The learner would like to restrict the amount of labeled data, which is significantly more expensive to obtain than unlabeled data.

In this paper, we show a gap between supervised and semi-supervised label complexities of adversarially robust learning in a distribution-free model. The label complexity in semi-supervised may be arbitrarily smaller compared to the supervised case, and is characterized by a different complexity measure.
Importantly, we are not using more data, just less labeled data.
The unlabeled sample size is the same as how much labeled data a fully-supervised method would require, and so this is a strict improvement.
This kind of gap is known not to hold in standard (non-robust) $\pac$ learning, this is
a unique property of robust learning.
    \paragraph{Background.}
    The following complexity measure $\VCU$ was introduced by \citet{montasser2019vc} (and denoted there by $\dim_{\U \times}$) as a candidate for determining the sample complexity of supervised robust learning. It was shown that indeed its finiteness is necessary, but not sufficient.
    This parameter is our primary object in this work, as we will show that it characterizes the labeled sample complexity of \textit{semi-supervised} robust $\pac$-learning.%

    \begin{definition}[$\VCU$-dimension]\label{def:vcu} A sequence of points  $\sett{x_1,\ldots,x_k}$ is $\U$-\textit{shattered} by $\H$ if $\forall y_1,\ldots,$ $y_k$ $\in $ $ \sett{0,1}$, $\exists h\in \H$ such that $\forall i\in [k],\forall z\in \U(x_i)$, $h(z)=y_i$. The $\VCU(\H)$ is largest integer $k$ for which there exists a sequence $\sett{x_1,\ldots,x_k}$ $\U$-shattered by $\H$.
    \end{definition}
    Intuitively, this dimension relates to shattering of the entire perturbation sets, instead of one point in the standard $\VC$-dimension.
    When $\U(x)=\sett{x}$, this parameter coincides with the standard $\VC$. Moreover, for any hypothesis class $\H$, it holds that $\VCU(\H)\leq \VC(\H)$, and the gap can be arbitrarily large. That is, there exist $\H_0$ such that $\VCU(\H_0)=0$ and $\VC(\H_0)=\infty$ (see Proposition \ref{prop:gap-vcu-dimu}).
    
     For an improved lower bound on the sample complexity, \citet[Theorem 10]{montasser2019vc} introduced the Robust Shattering dimension, denoted by $\mathrm{RS}_\U$ (and denoted there by $\dim_\U$). 
    \begin{definition}[$\mathrm{RS}_\U$-dimension]\label{def:dimu} A sequence $x_1,\ldots,x_k$ is said to be $\U\text{-robustly shattered}$ by $\F$ if \\
    $ \exists z^+_1,z^-_1,\ldots,z^+_k,z^-_k$ such that $x_i \in \U\paren{z^+_i}\cap \U\paren{z^-_i} \forall i\in[k]$ and $\forall y_1,\ldots,y_k \in \sett{+,-}, \exists f\in \F$ with $f(\zeta)=y_i$, $\forall \zeta\in \U\paren{z^{y_i}_{i}}, \forall i\in[k]$. The $\U$-\text{robust shattering dimension} $\mathrm{RS}_\U(\H)$ is defined as the maximum size of a set that is $\U$-robustly shattered by $\H$.
    \end{definition}
     Specifically, the lower bound on the sample complexity is
    $ \Omega\paren{\frac{\mathrm{RS}_\U}{\epsilon}+\frac{1}{\epsilon}\log\frac{1}{\delta}}$ for realizable robust learning, and $\Omega\paren{\frac{\mathrm{RS}_\U}{\epsilon^2}+\frac{1}{\epsilon^2}\log\frac{1}{\delta}}$
    for agnostic robust learning.
    They also showed upper bounds of
     $\Tilde{\O}\paren{\frac{\VC\cdot\VC^*}{\epsilon}+\frac{\log\frac{1}{\delta}}{\epsilon}}$\footnote{$\Tilde{\O}(\cdot)$ stands for omitting poly-logarithmic factors of $\VC,\VC^*,\VCU,\mathrm{RS}_\U,1/\epsilon,1/\delta$.} in the realizable case and $\Tilde{\O}\paren{\frac{\VC\cdot\VC^*}{\epsilon^2}+\frac{\log\frac{1}{\delta}}{\epsilon^2}}$ in the agnostic case, where $\VC^*$ is the dual $\VC$ dimension (definitions are in \cref{app:prelim}).
    \citet{montasser2019vc} showed that for any $\H$, $\VCU(\H)\leq \mathrm{RS}_\U(\H) \leq\VC(\H)$, and there can be an arbitrary gap between them. Specifically, there exists $\H_0$ with $\VC_\U(\H_0)=0$ and $\mathrm{RS}_\U(\H_0)=\infty$,
    and there exists $\H_1$ with $\mathrm{RS}_\U(\H_1)=0$ and $\VC(\H_1)=\infty$.
    \paragraph{Main contributions.}
    \begin{itemize}[leftmargin=0.3cm]
    \item In \cref{sec:knowing-support}, we first analyze the simple case where the support of the marginal distribution on the inputs is fully known to the learner. In this case, we show a tight bound of $\Theta\paren{\frac{\VCU(\H)}{\epsilon}+\frac{\log\frac{1}{\delta}}{\epsilon}}$ on the labeled complexity for learning $\H$.
    
    \item In \cref{sec:realizable}, we present a generic algorithm
    that can be applied both for the realizable and agnostic settings. We prove an upper bound and nearly matching lower bounds on the sample complexity in the realizable case.
    For semi-supervised robust learning, we prove a labeled sample complexity bound $\Lambda^{\mathrm{ss}}$ and compare to the sample complexity of supervised robust learning $\Lambda^{\mathrm{s}}$.  Our algorithm uses $\Lambda^{\mathrm{ss}} = \Tilde{\O}\paren{\frac{\VCU}{\epsilon}+\frac{1}{\epsilon}\log\frac{1}{\delta}}$
    \textit{labeled} examples and $\O(\Lambda^{\mathrm{s}})$ \textit{unlabeled}
    examples.  Recall that $\Lambda^{\mathrm{s}} =\Omega(\mathrm{RS}_\U)$, and since $\mathrm{RS}_\U$ can be arbitrarily larger than $\VCU$, 
    this means our labeled sample complexity represents a strong improvement over the sample complexity of supervised learning.
    \item In \cref{subsec:agnostic}, we prove upper and lower bounds on the sample complexity in the agnostic setting. We reveal an interesting structure, which is inherently different than the realizable case.
    Let $\eta$ be the minimal agnostic error. 
    If  we allow an error of $3\eta+\epsilon$, it is sufficient for our algorithm to have $\Lambda^{\mathrm{ss}}=\Tilde{\O}\paren{\frac{\VCU}{\epsilon^2}+\frac{\log\frac{1}{\delta}}{\epsilon^2}}$ \textit{labeled} examples and $\O(\Lambda^{\mathrm{s}})$ \textit{unlabeled} examples (as in the realizable case). If we insist on having error $\eta+\epsilon$, then there is a lower bound of $\Lambda^{\mathrm{ss}}=\Omega\paren{\frac{\mathrm{RS}_\U}{\epsilon^2}+\frac{1}{\epsilon^2}\log\frac{1}{\delta}}$ labeled examples.
    Furthermore, an error of $(\frac{3}{2}-\gamma)\eta+\epsilon$ is unavoidable if the learner is restricted to $\O(\VCU)$ labeled examples, for any $\gamma>0$.
    We also show that \textit{improper} learning is necessary, similar to the supervised case. 
    We summarize the results in \cref{fig:results} showing for which labeled and unlabeled samples we have a robust learner.
    
    \item The above results show that there is a significant benefit in semi-supervised robust learning. 
    For example, take $\H_0$ with $\VCU(\H_0)=0$ and $\mathrm{RS}_\U(\H_0)=n$. 
    The labeled sample size for learning $\H_0$ in supervised learning is $\Omega(n)$.
    In contrast, in semi-supervised learning our algorithms requires only $\O(1)$ \textit{labeled} examples and $\O(n)$ \textit{unlabeled} examples. We are not using more data, just less labeled data.
    Note that $n$ can be arbitrarily large.
    \item A byproduct of our result is that if we assume that the distribution is robustly realizable by a hypothesis class (i.e., there exist a hypothesis with zero robust error) then, with respect to the \underline{non-robust} loss (i.e., the standard $0$-$1$ loss) we can learn with only $\Tilde{\O}\paren{\frac{\VCU(\H)}{\epsilon}+\frac{\log\frac{1}{\delta}}{\epsilon}}$  labeled examples, even if the $\VC$ is infinite. Recall that there exists $\H_0$ with $\VC_\U(\H_0)=0$, $\mathrm{RS}_\U(\H_0)=\infty$ and $\VC(\H_0)=\infty$.
    Learning linear functions with margin is a special case of this data-dependent assumption. 
    Moreover, we show that this is obtained only by \textit{improper} learning. (See \cref{sec:improved-01-loss}.)
    \end{itemize}
    
    \begin{SCfigure}[]\label{fig:results}
    \centering
    \caption{Summary of the sample complexity regimes for semi-supervised robust learning, for the realizable model and the agnostic model with error $3\eta+\epsilon$, where $\eta$ is the minimal agnostic error in the hypothesis class.
    
     Obtaining an error of $\eta+\epsilon$ 
     requires at least $\mathrm{RS}_\U$ labeled examples, as in the supervised case. 
     
     $\Lambda^\mathrm{s}$ denotes the sample complexity of supervised robust learning. It is an open question whether  $\Lambda^\mathrm{s}$ equals $\mathrm{RS}_\U$.
    }
    \includegraphics[width=0.47\textwidth]%
    {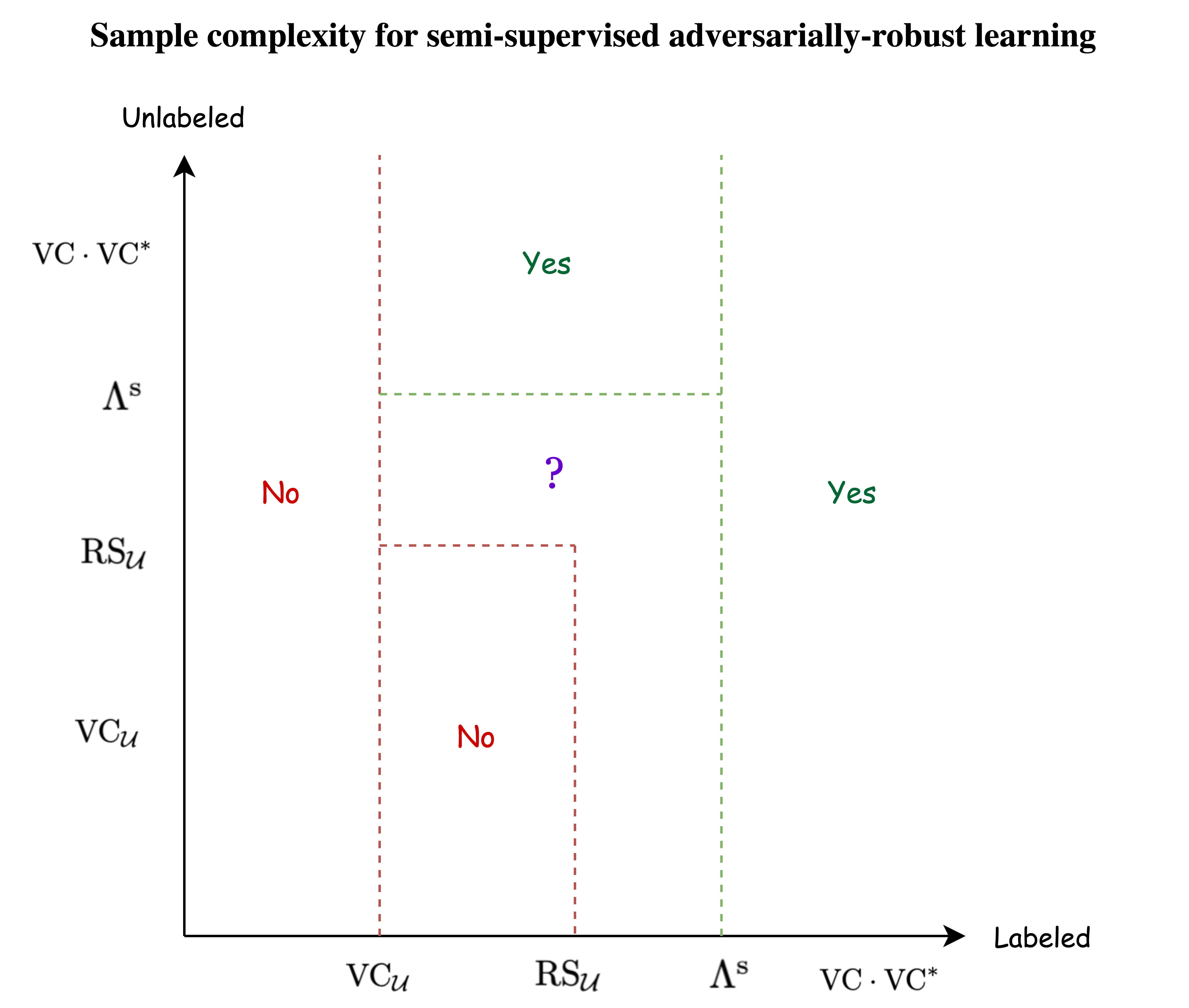}
    \end{SCfigure}
      \paragraph{Related work.}
      \noindent{\it Adversarially robust learning.}
      The work of \citet{montasser2019vc} studied the setting of fully-supervised robust $\pac$ learning. In this paper, we propose a semi-supervised method with a significant improvement on the labeled sample size. We show that the labeled and unlabeled sample complexities are controlled by different complexity measures.
      Adversarially robust learning has been extensively studied in several supervised learning models
    \cite[e.g.,][]{feige2015learning,schmidt2018adversarially,khim2018adversarial,yin2019rademacher,cullina2018pac,attias2019improved,kontorovich2021fat,attias2022adversarially,awasthi2021calibration,montasser2021transductive,montasser2020efficiently,montasser2020reducing,montasser2021adversarially,ashtiani2020black,awasthi2021existence,dan2020sharp,awasthi2020adversarial,bhattacharjee2021sample,xing2021adversarially,ashtiani2022adversarially}.
    For semi-supervised robust learning, \citet{ashtiani2020black} showed that under some assumptions, robust $\pac$ learning is possible with $\O\lr{\VC(\H)}$ labeled examples and additional unlabeled samples.
     \citet{carmon2019unlabeled} studied a robust semi-supervised setting where the distribution is a mixture of Gaussians and the hypothesis class is linear separators.

      \noindent{\it Semi-supervised (non-robust) learning.}
      There is substantial interest
        in semi-supervised (non-robust) learning, and many contemporary practical problems significantly benefit from it \cite[e.g.,][]{DBLP:reference/algo/Blum16,chapelle2009semi,zhu2009introduction}.
        This was formalized in theoretical frameworks.
    \citet{urner2011access} suggested a semi-supervised learning (non-robust) framework, with an algorithmic idea that is similar to our method. Their framework consists of two steps; using labeled data to learn a classifier with small error (not necessarily a member of the target class $\H$), and then labeling an unlabeled input sample in order to use a fully-supervised proper learner. They investigate scenarios where saving of labeled examples occurs. 
    In our paper, we are interested in the robust loss function. We use labeled data in order to learn a classifier (with the 0-1 loss function) from a class with a potentially smaller complexity measure, then we label an unlabeled input sample, and use a fully-supervised method using the robust loss function. The sample complexity of learning the robust loss class is controlled by a larger complexity measure. Fortunately, this affects our unlabeled sample size and not the labeled sample size as in the fully-supervised setting.
    \citet{gopfert2019can} studied circumstances where the learning rate can be improved given unlabeled data. \citet{darnstadt2013unlabeled} showed that the label complexity gap between the semi-supervised and the fully supervised setting
    can become arbitrarily large for concept classes of infinite $\VC$-dimension, and that this gap is bounded when a function class contains the constant zero and the constant one functions.
    \citet{balcan2010discriminative,balcan200621} introduced an 
    augmented version of the $\pac$ model designed for semi-supervised learning and analyzed when unlabeled data can help. The main idea is to augment the notion of learning a concept class, with a notion of compatibility between a function and the data distribution that we hope the target function will satisfy. 
    \section{Preliminaries}\label{sec:prelim}
    Let $\X$ be the instance space, $\Y$ a label space, and  $\H\subseteq\Y^{\X}$ a hypothesis class.
    A perturbation function $\U: \X \rightarrow 2^{\X}$ maps an input to a set $\U(x)\subseteq \X$. 
    Denote the 0-1 loss of hypothesis $h$ on $(x,y)$ by $\ell_{0\text{-}1}(h;x,y)=\I\sqparen{h(x)\neq y}$, and the robust loss with respect to $\U$ by $\ell_{\U}(h;x,y)=\underset{z\in \U(x)}{\sup}\I\sqparen{h(z)\neq y}$.
    Denote the support of a distribution $\D$ over $\X\times\Y$ by $\supp(\D)=\sett{(x,y)\in \X\times\Y:\D(x,y)>0}$. Denote the marginal distribution  $\D_{\X}$ on $\X$ and its support by $\supp(\D_{\X})=\sett{x\in\X:\D(x,y)>0}$.
     Define the \textit{robust risk} of a hypothesis $h\in \H$ with respect to distribution $\D$ over $\X \times \Y$,
    $$\risk_{\U}\paren{h;\D}=\E_{\paren{x,y}\sim \D}\sqparen{\ell_\U(h;x,y)}=\E_{\paren{x,y}\sim \D}\sqparen{\sup_{z\in \U(x)}\I\sqparen{h(z)\neq y}}.$$
    The approximation error of $\H$ on $\D$, namely, the optimal robust error achievable by a hypothesis in $\H$ on $\D$ is denoted by,
    $$\risk_{\U}(\H;\D)=\inf_{h\in\H}\risk_{\U}\paren{h;\D}.$$
    We say that a distribution $\D$ is \textit{robustly realizable} by a class $\H$ if $\risk_{\U}(\H;\D)=0$.
    
    Define the \textit{empirical robust risk} of a hypothesis $h\in \H$ with respect to a sequence $S\in  \paren{\X\times \Y}^*,$ 
    $$\widehat{\risk}_{\U}\paren{h;S}
    =
     \frac{1}{|S|}\sum_{\paren{x,y}\in S}\ell_\U(h;x,y)
    =
    \frac{1}{|S|}\sum_{\paren{x,y}\in S}\sqparen{\sup_{z\in \U(x)}\I\sqparen{h(z)\neq y}}.$$
    The \textit{robust empirical risk minimizer} learning algorithm $\RERM:\paren{\X\times\Y}^*\rightarrow \H$ for a class $\H$ on a sequence $S$ is defined by
    $$\RERM_\H(S)\in\argmin_{h\in\H}\widehat{\risk}_{\U}\paren{h;S}.$$
    
    When the perturbation function is the identity, $\U(x)=\{x\}$, we recover the standard notions. The \textit{risk} of a hypothesis $h\in \H$ with respect to distribution $\D$ over $\X \times \Y$ is defined by
    $\risk\paren{h;\D}=\E_{\paren{x,y}\sim \D}\sqparen{\ell_{0\text{-}1}(h;x,y)}=\E_{\paren{x,y}\sim \D}\sqparen{\I\sqparen{h(x)\neq y}},$ 
    and the \textit{empirical risk} of a hypothesis $h\in \H$ with respect to a sequence $S\in  \paren{\X\times \Y}^*$ is defined by 
    $\widehat{\risk}\paren{h;S}
    =
     \frac{1}{|S|}\sum_{\paren{x,y}\in S}\ell_{0\text{-}1}(h;x,y)
    =
    \frac{1}{|S|}\sum_{\paren{x,y}\in S}\sqparen{\I\sqparen{h(x)\neq y}}.$
    The \textit{empirical risk minimizer} learning algorithm $\ERM:\paren{\X\times\Y}^*\rightarrow \H$ for a class $\H$ on a sequence $S$ is defined by
    $\ERM_\H(S)\in\argmin_{h\in\H}\widehat{\risk}\paren{h;S}.$

    A learning algorithm $\A :\paren{\X\times\Y}^*\rightarrow \Y^{\X}$ for a class $\H$ is called \textit{proper} if it always outputs a hypothesis in $\H$, otherwise it is called \textit{improper}.
   
   \paragraph{Realizable robust $\pac$ learning.}We define the supervised and semi-supervised settings.
   \begin{definition}
    [Realizable robust $\pac$ learnability] For any $\epsilon,\delta\in (0,1)$, the sample complexity of realizable robust $(\epsilon,\delta)$-$\pac$ learning for a class $\H$, with respect to perturbation function $\U$, denoted by $\Lambda_{\RE}(\epsilon,\delta,\H,\U)$, is the smallest integer $m$ for which there exists a learning algorithm $\A :\paren{\X\times\Y}^*\rightarrow \Y^{\X}$, such that for every distribution $\D$ over $\X\times\Y$ robustly realizable by $\H$, namely $\risk_{\U}\paren{\H;D}=0$, for a random sample $S \sim \D^m$, it holds that 
    $$\Pr\paren{\risk_\U\paren{\A(S);D}\leq \epsilon}>1-\delta.$$
    If no such $m$ exists, define $\Lambda_{\RE}(\epsilon,\delta,\H,\U) = \infty$,
    and $\H$ is not robustly $(\epsilon,\delta)$-$\pac$ learnable with respect to $\U$.
   \end{definition}
    For the standard (non-robust) learning with the 0-1 loss function, we omit the dependence on $\U$ and denote the sample complexity of class $\H$ by $\Lambda_{\RE}(\epsilon,\delta,\H)$.
    \begin{definition}[Realizable semi-supervised robust $\pac$ learnability]
    A hypothesis class $\H$ is semi-supervised realizable robust $(\epsilon,\delta)$-$\pac$ learnable, with respect to perturbation function $\U$,  if for any $\epsilon,\delta\in (0,1)$, there exists  $m_{u},m_{l}\in \N \cup \sett{0}$, and a learning algorithm $\A :\paren{\X\times\Y}^* \cup \paren{\X}^* \rightarrow \Y^{\X}$, such that for every distribution $\D$ over $\X\times\Y$ robustly realizable by $\H$, namely $\risk_{\U}\paren{\H;D}=0$, for random samples $S^l \sim \D^{m_l}$ and $S^{u}_{\X} \sim \D_{\X}^{m_u}$, it holds that 
    $$\Pr\paren{\risk_\U\paren{\A(S^l,S^{u}_{\X});D}\leq \epsilon}>1-\delta.$$
    \end{definition}
    The sample complexity $\M_{\RE}{\paren{\epsilon,\delta,\H,\U}}$ includes all such pairs $(m_u,m_l)$. 
    If no such $(m_u,m_l)$ exist, then $\M_{\RE}(\epsilon,\delta,\H,\U) = \emptyset$. 
    \paragraph{Agnostic robust $\pac$ learning.}In this case we have $\risk_{\U}\paren{\H;\D}>0$, and we would like to compete with the optimal $h\in\H$.
    We add a parameter to the sample complexity, denoted by $\eta$, which is the optimal robust error of a hypothesis in $\H$, namely $\eta = \risk_{\U}\paren{\H;\D}$.
    We say that a function $f$ is $(\alpha,\epsilon)$-optimal if $\risk_\U\paren{f;D}\leq \alpha\eta+\epsilon$.
      
    \begin{definition}[Agnostic robust $\pac$ learnability]
    For any $\epsilon,\delta\in (0,1)$, the sample complexity of agnostic robust $(\alpha,\epsilon,\delta)$-$\pac$ learning 
    for a class $\H$, with respect to perturbation function $\U$, denoted by $\Lambda_{\AG}(\alpha,\epsilon,\delta,\H,\U,\eta)$, is the smallest integer $m$, for which there exists a learning algorithm $\A:\paren{\X\times\Y}^* \rightarrow \Y^{\X}$, such that for every distribution $\D$ over $\X\times\Y$, for a random sample $S\sim \D^m$, it holds that
    $$\Pr\paren{\risk_\U\paren{\A(S);D}\leq \alpha\inf_{h\in\H}\risk_{\U}\paren{h;\D} + \epsilon} >1-\delta.$$

    If no such $m$ exists, define $\Lambda_{\AG}(\alpha,\epsilon,\delta,\H,\U,\eta) = \infty$, and $\H$ is not robustly $(\alpha,\epsilon,\delta)$-$\pac$ learnable in the agnostic setting with respect to $\U$.
    Note that for $\alpha=1$ we recover the standard agnostic definition, our notation allows for a more relaxed approximation.
    \end{definition}  
    Analogously, we define the semi-supervised case.
    \begin{definition}[Agnostic semi-supervised robust $\pac$ learnability]
    A hypothesis class $\H$ is semi-supervised agnostically robust $(\alpha,\epsilon,\delta)$-$\pac$ learnable, with respect to perturbation function $\U$, if for any $\epsilon,\delta\in (0,1)$, there exists  $m_{u},m_{l}\in \N \cup \sett{0}$, and a learning algorithm $\A :\paren{\X\times\Y}^* \cup \paren{\X}^* \rightarrow \Y^{\X}$, such that for every distribution $\D$ over $\X\times\Y$, for random samples $S^l \sim \D^{m_l}$ and $S^{u}_{\X} \sim \D_{\X}^{m_u}$, it holds that 
    
    $$\Pr\paren{\risk_\U\paren{\A(S^l,S^{u}_{\X});\D}\leq \alpha\inf_{h\in\H}\risk_{\U}\paren{h;\D}+\epsilon}>1-\delta.$$
    
    The sample complexity $\M_{\AG}(\alpha,\epsilon,\delta,\H,\U,\eta)$ includes all such pairs $(m_u,m_l)$. 
    If no such $(m_u,m_l)$ exist, then $\M_{\AG}(\alpha,\epsilon,\delta,\H,\U,\eta) = \emptyset$.
    \end{definition}
    \paragraph{Partial concept classes \citep{alon2021theory}.} Let a partial concept class $\H\subseteq\sett{0,1,\star}^\X$. For $h\in\H$ and input $x$ such that $h(x)=\star$, we say that $h$ is undefined on $x$.
    The support of a partial hypothesis $h:\X\rightarrow \sett{0,1,\star}$ is the preimage of $\sett{0,1}$, formally, $h^{-1}(\sett{0,1})=\sett{x\in\X: h(x)\neq \star}$.
    The main motivation of introducing partial concepts classes, is that data-dependent assumptions can be modeled in a natural way that extends the classic theory of total concepts.
    The $\VC$ dimension of a partial class $\H$ is defined as the maximum size of a shattered set $S\subseteq \X$, where $S$ is shattered by $\H$ if the projection of $\H$ on $S$ contains all possible binary patterns, $\sett{0,1}^S\subseteq \H|_{S}$.
    The $\VC$-dimension also characterizes verbatim the $\pac$ learnability of partial concept classes, even though uniform convergence does not hold in this setting.

    We use the notation $\Tilde{\O}(\cdot)$ for omitting poly-logarithmic factors of $\VC,\VC^*,\VCU,\mathrm{RS}_\U,1/\epsilon,1/\delta$.
    See \cref{app:prelim} for additional preliminaries on complexity measures, sample compression schemes, and partial concept classes.
    
    \section{Warm-up: knowing the support of the marginal distribution}\label{sec:knowing-support}
    
    In this section, we provide a tight bound on the labeled sample complexity when the support of marginal distribution is fully known to the learner, under the robust realizable assumption.
    Studying this setting gives an intuition for the general semi-supervised model.
    The main idea is that as long as we know the support of the marginal distribution, $\supp(\D_{\X})=\sett{x\in\X:\exists y\in \Y, \text{ s.t.  } \D(x,y)>0}$, we can restrict our search to a subspace of functions that are robustly self-consistent, $\H_{\U\text{-cons}}\subseteq \H$, where
    $$
    \H_{\U\text{-cons}}=\sett{h\in \H: \forall x\in \supp(\D_\X),
    \forall z,z'\in \U(x), h(z)=h(z')}
    .$$ 
    
    As long as the distribution is robustly realizable, i.e., $\risk_{\U}(\H;\D)=0$, we are guaranteed that the target hypothesis belongs to $\H_{\U\text{-cons}}$. As a result, it suffices to learn the class $\H_{\U\text{-cons}}$ with the 0-1 loss function, in order to robustly learn the original class $\H$. 
    We observe that,
    $$
    \VC(\H_{\U\text{-cons}})= \VCU(\H) \leq \VC(\H).
    $$
    Moreover, there exits $\H_0$ with $\VCU(\H_0)=0$ and $\VC(\H_0)=\infty$ (see Proposition \ref{prop:gap-vcu-dimu}). Fortunately, moving from $\VC(\H)$ to $ \VCU(\H)$ implies a significant sample complexity improvement.
    Since $\supp(\D_\X) $ is known, we can now employ any algorithm for learning the hypothesis class $\H_{\U\text{-cons}}$.
    \footnote{See \citet[Chapter 3]{mohri2018foundations} for standard upper and lower bounds. In order to remove the superfluous $\log\frac{1}{\epsilon}$ factor of the standard uniform convergence based upper bound, $\O\paren{\frac{\VCU(\H)}{\epsilon}\log\frac{1}{\epsilon}+\frac{\log\frac{1}{\delta}}{\epsilon}}$, we can use the learning algorithm and its analysis 
    from \citet{hanneke2016optimal} that applies for any $\H$ and $\D$, or some other algorithms that are doing so while restricting the hypothesis class or the data distribution \citep[e.g.,][]{auer2007new,darnstadt2015optimal,hanneke2016refined,hanneke2009theoretical,long2003upper,gine2006concentration,bshouty2009using,balcan2013active}.}
    This leads eventually to robustly learn $\H$ with labeled sample complexity that scales linearly with $\VCU$ (instead of the $\VC$). Formally,
    \begin{theorem}\label{thm:known-marginal}
    For hypothesis class $\H$ and adversary $\U$, when the support of the marginal distribution $\D_{\X}$ is known, 
    the labeled sample complexity is $\Theta\paren{\frac{\VCU(\H)}{\epsilon}+\frac{\log\frac{1}{\delta}}{\epsilon}}$.
    \end{theorem}
The following Proposition demonstrates that  semi-supervised robust learning requires much less labeled samples compared to the supervised counterpart. Recall the lower bound on the sample complexity of supervised robust learning, $ \Lambda_{\RE}(\epsilon,\delta,\H,\U)= \Omega\paren{\frac{\mathrm{RS}_\U(\H)}{\epsilon}+\frac{1}{\epsilon}\log\frac{1}{\delta}}$ given by \citet[Theorem 10]{montasser2019vc}. For completeness, we prove the following in \cref{app:knowing-support}.
    \begin{proposition}[\cite{montasser2019vc}, Proposition 9]\label{prop:gap-vcu-dimu}
    There exists a hypothesis class $\H_0$ such that $\VCU(\H_0)=0$, $\mathrm{RS}_\U(\H_0) = \infty$, and $\VC(\H_0) = \infty$.
\end{proposition}
We can now conclude the following separation result on supervised and semi-supervised label complexities.
    \begin{corollary}
    The hypothesis class in Proposition \ref{prop:gap-vcu-dimu} is not learnable in supervised robust learning (i.e., we need to see the entire data distribution).
    However, when $\supp(\D_{\X})$ is known, this class can be learned with $\O(\frac{1}{\epsilon}\log\frac{1}{\delta})$ labeled examples.
    \end{corollary}
    In the next section, we prove a stronger separation in the general semi-supervised setting. The size of the labeled data required in the supervised case is lower bounded by $\mathrm{RS}_\U$, whereas in the semi-supervised case the \textit{labeled} sample complexity depends only on $\VCU$ and the \textit{unlabeled} data is lower bounded by $\mathrm{RS}_\U$.
    Moreover, note that in \cref{thm:known-marginal}, when $\supp(\D_{\X})$ is known,  we can use any proper learner. In \cref{sec:realizable} we show that in the general semi-supervised model this is not the case, and sometimes improper learning is necessary, similarly to supervised robust learning.
    
    \section{Near-optimal semi-supervised sample complexity}\label{sec:realizable}
    In this section we present our algorithm and its guarantees for the realizable setting. We also prove nearly matching lower bounds on the sample complexity. Finally, we show that improper learning is necessary in semi-supervised robust learning, similar to the supervised case. 
    
    We present a generic semi-supervised robust learner, that can be applied on both realizable and agnostic settings. The algorithm uses the following two subroutines. The first one is any algorithm for learning partial concept classes, which controls our \textit{labeled} sample size. (In \cref{app:algo-partial} we discuss in detail the algorithm suggested by \citet{alon2021theory}.) The second subroutine, is any algorithm for the agnostic adversarially robust supervised learning, which controls our \textit{unlabeled} sample size. (In  \cref{app:algo-agnostic-robust}
    we discuss in detail the algorithm suggested by \citet{montasser2019vc}.)
    Any progress on one of these problems improves directly the guarantees of our algorithm.
    We use the following definition that explains how to convert a total concept class into a partial one, in a way that preserves the idea of the robust loss function.

    \begin{definition}\label{def:partial-robust-class}
    Let a hypothesis class $\H\subseteq \sett{0,1}^{\X}$ and a perturbation function $\U: \X \rightarrow 2^{\X}$.
    For any $h\in\H$, we define a corresponding partial concept $h^\star:\X \rightarrow \sett{0,1,\star}$, and denote this mapping by $\phi(h)=h^\star$. For $x\in \X$, whenever $h$ is not consistent on the entire set $\,\U(x)$, i.e., $\exists z,z'\in \U(x), h(z)\neq h(z')$, define $h^\star(x)=\star$. Otherwise, $h$ is robustly self-consistent on $x$, i.e., $\forall z,z'\in \U(x), h(z)=h(z')$ and $h$ remains unchanged, $h^\star(x)=h(x)$. The corresponding partial concept class is defined by $\H^{\star}_{\U}= \sett{h^{\star}: \phi(h)=h^\star,\; \forall h \in \H}$.
    \end{definition}
    The main motivation for the above definition is the following. Fix a hypothesis $h$.
    For any point $x$, as defined above, the adversary can force a mistake on $h$, regardless of the prediction of $h$.
    We would like to mark such points as \emph{mistake}.
    We do this by defining a partial concept $h^\star$ and setting $h^\star(x)=\star$, which, for partial concepts, implies a mistake. The benefit of this preprocessing is that we reduce the complexity of the hypothesis class from $\VC$ to $\VCU$, which potentially can reduce the labeled sample complexity.

    We are now ready to describe the algorithm.
      \begin{algorithm}[H]
    \caption{Generic Adversarially-Robust Semi-Supervised ($\grass$) learner}\label{alg:generic-algo}
    \textbf{Input:} Labeled data set $S^{l}\sim \D^{m_l}$, unlabeled data set $S^{u}_{\X}\sim \D^{m_u}_{\X}$, hypothesis class $\H$, perturbation function $\U$, parameters $\epsilon$, $\delta$.\\
    \textbf{Algorithms used:} 
    $\pac$ learner $\A$ for partial concept classes,
    agnostic adversarially robust \underline{supervised} $\pac$ learner $\B$.
    \begin{enumerate}
        \item Given the class $\H$, construct the hypothesis class $\H^{\star}_\U$ using Definition \ref{def:partial-robust-class}.
        \item Execute the learning algorithm for partial concepts $\A$ on $\H^{\star}_\U$ and sample ${S^l}$,  with the $0$-$1$ loss and parameters $\frac{\epsilon}{3},\frac{\delta}{2}$. 
        Denote the resulting hypothesis $h_1$.
        
        \item Label the unlabeled data set $S^{u}_{\X}$ with $h_1$, denote the labeled sample by $S^u$.
        (On points where $h_1$ predicts $\star$, we can arbitrarily choose a label of $0$ or $1$.)
        \item Execute the agnostic adversarially robust supervised $\pac$ learner $\B$ 
        on $S^u$ with parameters $\frac{\epsilon}{3},\frac{\delta}{2}$.
        Denote the resulting hypothesis $h_2$.
    \end{enumerate}
    \textbf{Output:} $h_2$.
\end{algorithm}

    \paragraph{Algorithm motivation.} 
    The main idea behind the algorithm is the following. Given the class $\H^{\star}_\U$, we would like to find a hypothesis $h_1\in \H^{\star}_\U$ which has a small error, whose existence follows from our realizability assumption. The required sample size scales with $\VCU$, which is the complexity of $\H^{\star}_\U$, rather than $\VC$. This is where we make a significant gain in the labeled sample complexity. Note that $h_1$ does not guarantee a small robust error, although it does guarantee a small non-robust error. 
    We utilize an additional unlabeled sample for this task, which we label using $h_1$. If we would simply minimize the non-robust error on this sample we would simply get back $h_1$. The main insight is that we would like to minimize the robust error over this sample, which will result in hypothesis $h_2$. We now need to bound the robust error of $h_2$. The optimal function $h_{\mathrm{opt}}$ has only a slightly increased robust error on this sample, namely, at most on the sample points where it disagrees with $h_1$. Note that $h_1$ might have a large robust error due to the perturbation $\U$.
    However, a robust supervised PAC learner would return a hypothesis $h_2$ which has robust error similar to $h_{\mathrm{opt}}$, which is at most $\epsilon$.
    
    \paragraph{Algorithm outline and guarantees.} 
    In the first step, we convert $\H$ to $\H^{\star}_\U$. Then we employ a learning algorithm $\A$ for partial concepts on $\H^{\star}_\U$ with a labeled sample $S^l\sim \D^{m_l}$. The output of the algorithm is a function $h_1$ with $\epsilon/3$ on the \underline{0-1} error. Crucially, we needed for this step $|S^l|=\Tilde{\O}\lr{\VCU(\H)/\epsilon}$ labeled examples for learning the partial concept $\H^{\star}_\U$, since $\VC(\H^{\star}_\U)=\VCU(\H)$. So our labeled sample size is controlled by the sample complexity for learning partial concepts with the 0-1 loss.
In step 3, we label an independent unlabeled sample $S^{u}_{\X}\sim \D^{m_u}_{\X}$ with $h_1$, denote his labeled sample by $S^u$. Define a distribution $\Tilde{\D}$ over $\X\times\Y$ by $\Tilde{\D}(x,h_1(x))=\D_{\X}(x),$
    and so $S^u$ is an i.i.d. sample from $\Tilde{\D}$.
    We argue that the robust error of $\H$ with respect to $\Tilde{\D}$ is at most $\frac{\epsilon}{3}$, i.e., $\risk_{\U}(\H;\Tilde{\D})=\frac{\epsilon}{3}$. Indeed, the function with zero robust error on $\D$,  $h_{\text{opt}}\in\argmin_{h\in\H}\risk_{\U}(h;\D)$ has a robust error of at most $\frac{\epsilon}{3}$ on $\Tilde{\D}$.
    Finally, we employ an agnostic adversarially robust \underline{supervised} $\pac$ learner $\B$ for the class $\H$ on $S^u\sim \Tilde{\D}^{m_u}$, that should be of size of the sample complexity of agnostically robust learn $\H$ with respect to $\U$, when the optimal robust error of hypothesis from $\H$ on $\Tilde{\D}$ is at most $\frac{\epsilon}{3}$. 
    Moreover, the total variation distance between $\D$ and $\Tilde{\D}$ is at most $\frac{\epsilon}{3}$. 
    We are guaranteed that the resulting hypothesis $h_2$ has a \underline{robust} error of at most $\frac{\epsilon}{3}+\frac{\epsilon}{3}+\frac{\epsilon}{3} =\epsilon$ on $\D$. We conclude that a size of $|S^{u}_{\X}|=m_u =\Lambda_{\AG}\paren{1,\frac{\epsilon}{3},\frac{\delta}{2},\H,\U,\eta=\frac{\epsilon}{3}}$ unlabeled samples suffices, this completes the proof for \cref{thm:realizable-sample-compexity}. For a specific instantiation of such algorithm (\cite{montasser2019vc}), we deduce the sample complexity in \cref{thm:realizable-sample-compexity-cor}.
    A simple analysis of the latter yields a dependence of $\epsilon^2$ for the unlabeled sample size. However, by applying a suitable data-dependent generalization bound, we reduce this dependence to $\epsilon$. (Full proofs appear in \cref{app:realizable}).
    
    We now formally present the sample complexity of the generic semi-supervised learner for the robust realizable setting. First, in the case of using a generic agnostic robust supervised learner as a subroutine (step 4 in the algorithm). Then we deduce the sample complexity of a specific instantiation of such algorithm. 
    \begin{theorem}\label{thm:realizable-sample-compexity}
    For any hypothesis class $\H$ and adversary $\U$,
    algorithm $\grass$
    $(\epsilon,\delta)$-$\pac$ learns $\H$ with respect to the robust loss function, in the realizable robust  case, with samples of size
        \begin{align*}
            m_l= \O\paren{\frac{\VCU(\H)}{\epsilon}\log^2\frac{\VCU(\H)}{\epsilon}+\frac{\log\frac{1}{\delta}}{\epsilon}}
            \;,\;
            m_u = \Lambda_{\AG}\paren{1,\frac{\epsilon}{3},\frac{\delta}{2},\H,\U,\eta=\frac{\epsilon}{3}},
        \end{align*}
    where $\Lambda_{\AG}\paren{\alpha,\epsilon,\delta,\H,\U,\eta}$ is the sample complexity of adversarially-robust agnostic supervised $(\alpha,\epsilon,\delta)$-$\pac$ learning, such that $\eta$ is the error of the optimal hypothesis in $\H$, i.e., $\eta = \risk_{\U}\paren{\H;\D}$.
    \end{theorem}
    \begin{remark}
    Note that if we simply invoke a $\pac$ learner (for total concept classes) on $\H$, with the 0-1 loss, instead of steps 1 and 2 in the algorithm, we would get a labeled sample complexity of roughly $\O\lr{\VC(\H)}$. This is already an exponential improvement upon previous results that require roughly $\O\lr{2^{\VC(\H)}}$ labeled samples. The purpose of using partial concept classes is to further reduce the labeled sample complexity to $\O\lr{\VCU(\H)}$.
    \end{remark}
    The following result follows by using the agnostic supervised robust learner suggested by \citet{montasser2019vc}.
    A simple analysis of the latter yields a dependence of $\epsilon^2$ for the unlabeled sample size. However, by applying a suitable data-dependent generalization bound, we reduce this dependence to $\epsilon$. 
    \begin{theorem}\label{thm:realizable-sample-compexity-cor}
        For any hypothesis class $\H$ and adversary $\U$,
        Algorithm $\grass$
        $(\epsilon,\delta)$-$\pac$ learns $\H$ with respect to the robust loss function, in the realizable robust  case, with samples of size
            \begin{align*}
            m_l= \O\paren{\frac{\VCU(\H)}{\epsilon}\log^2\frac{\VCU(\H)}{\epsilon}+\frac{\log\frac{1}{\delta}}{\epsilon}}
            \;,\;
            m_u = \Tilde{\O}\paren{\frac{\VC(\H)\VC^*(\H)}{\epsilon}+\frac{\log\frac{1}{\delta}}{\epsilon}}.
        \end{align*}
    \end{theorem}
   We present nearly matching lower bounds for the realizable setting.
    The following Corollary stems from \cref{thm:known-marginal} and \citet[Theorem 10]{montasser2019vc}.
    \begin{corollary}
    For any $\epsilon,\delta\in (0,1)$, the sample complexity of realizable robust $(\epsilon,\delta)$-$\pac$ learning for a class $\H$, with respect to perturbation function $\U$ is
  \begin{align*}
    m_l= \Omega\paren{\frac{\VCU(\H)}{\epsilon}+\frac{\log\frac{1}{\delta}}{\epsilon}}
        \;,\;
         m_u = \infty,\;\;\; or  
    \;\;\; m_l+m_u= \Omega\paren{\frac{\mathrm{RS}_\U(\H)}{\epsilon}+\frac{\log\frac{1}{\delta}}{\epsilon}}.  
  \end{align*}
    \end{corollary}
\paragraph{Proper vs. improper.}
In Section \ref{sec:knowing-support}, we have seen that when the support of the marginal distribution $\D_{\X}$ is known, the labeled sample complexity is $\Theta\paren{\frac{\VCU(\H)}{\epsilon}+\frac{\log\frac{1}{\delta}}{\epsilon}}$. This was obtained by a proper learner: keep the robustly self-consistent hypotheses, $\H_{\U\text{-cons}}\subseteq \H$, and then use ERM on this class. The case when $\D_{\X}$ is unknown is different. We know that there exists a perturbation function $\U$ and a hypothesis class $\H$ with finite $\VC$-dimension that cannot be robustly $\pac$ learned with any proper learning rule \citep[Lemma 3]{montasser2019vc}. The same proof holds in the semi-supervised case.
Note that both algorithms $\A$ and $\B$ used in \cref{alg:generic-algo} are improper. (The proof appears in \cref{app:realizable}.)
\begin{theorem}\label{thm:improper}
There exists $\H$ with $\VC(\H)=0$ such that for any proper
learning rule $\A :\paren{\X\times\Y}^* \cup \paren{\X}^* \rightarrow \H$, there exists a distribution $\D$ over $\X \times \Y$ that is robustly realizable by $\H$, i.e., $\risk_{\U}\paren{\H;\D}=0$. 
It holds that
$\risk_{\U}\paren{\A(S^l,S^{u}_{\X});D}>\frac{1}{8}$ with probability at least $\frac{1}{7}$ over $S^l \sim \D^{m_l}$ and $S^{u}_{\X} \sim \D^{m_u}$, where $m_l,m_u\in \N \cup \sett{0}$ is the size of the labeled and unlabeled samples respectively.
Moreover, when the marginal distribution $\D_{\X}$ is known, there exists a proper learning rule for any $\H$.
\end{theorem} 
    \section{Agnostic robust learning}\label{subsec:agnostic}
    In this section, we prove the guarantees of \cref{alg:generic-algo} in the more challenging agnostic robust setting. We then prove lower bounds on the sample complexity which exhibit that it is inherently different from the realizable case.

     We follow the same steps as in the proof of the realizable case, with the following important difference. In the first two steps of the algorithm, we learn a partial concept class with respect to the $0$-$1$ loss, and obtain a hypothesis with error of $\eta+\epsilon/3$ ($\eta$ is the optimal robust error of a hypothesis in $\H$ and not $0$). This leads eventually to error of $\,3\eta+\epsilon\,$ for learning with respect to the robust loss.
     
     We then present two negative results.
     In \cref{thm:dimu-labels-agnostic} we show that for obtaining error $\eta+\epsilon$ there is a lower bound of $\Omega(\mathrm{RS}_\U)$ labeled examples, this result coincides with the lower bound of supervised robust learning. 
     In \cref{thm:vcu-2eta}, we show that 
      for any $\gamma>0$ there exist a hypothesis class, such that having access only to $\O(\VCU)$ labeled examples, leads to an error $(\frac{3}{2}-\gamma)\eta+\epsilon$. 
      (All proofs for this section are in \cref{app:agnostic}.)
     
     We start with the upper bounds. First, we analyze the case of using a generic agnostic robust learner, then we deduce the sample complexity of a specific instantiation of such algorithm.
     
    \begin{theorem}\label{thm:agnostic-sample-compexity}
    For any hypothesis class $\H$ and adversary $\U$, Algorithm $\grass$ $(3,\epsilon,\delta)\text{-}\pac$ learns $\H$ with respect to the robust loss function, in the agnostic robust case, with samples of size
        \begin{align*}
            m_l= \O\paren{\frac{\VCU(\H)}{\epsilon^2}\log^2\frac{\VCU(\H)}{\epsilon^2}+\frac{\log\frac{1}{\delta}}{\epsilon^2}}
            , \; \; 
            m_u = \Lambda_{\AG}\paren{1,\frac{\epsilon}{3},\frac{\delta}{2},\H,\U,2\eta+\frac{\epsilon}{3}}
            ,
        \end{align*}
        where $\Lambda_{\AG}\paren{\alpha,\epsilon,\delta,\H,\U,\eta}$ is the sample complexity of adversarially-robust agnostic supervised learning, such that $\eta$ is error of the optimal hypothesis in $\H$, namely $\eta = \risk_{\U}\paren{\H;\D}$. 
    \end{theorem}
    By using the agnostic supervised robust learner suggested by \citet{montasser2019vc}, we have the following upper bound on the unlabeled sample size, $m_u
    =
    \Tilde{\O}\paren{\frac{\VC(\H)\VC^*(\H)}{\epsilon^2}+\frac{\log\frac{1}{\delta}}{\epsilon^2}}$.

    We now present two negative results.
    
    \begin{theorem}\label{thm:dimu-labels-agnostic}
    For any $\epsilon,\delta\in (0,1)$, the sample complexity of agnostic robust $(1,\epsilon,\delta)$-$\pac$ learning 
    for a class $\H$, with respect to perturbation function $\U$ is (even if $\D_\X$ is known),
    \begin{align*}
        m_l = \Omega\paren{\frac{\mathrm{RS}_\U(\H)}{\epsilon^2}+\frac{1}{\epsilon^2}\log\frac{1}{\delta}}
        \;,\;
        m_u = \infty.
    \end{align*}
    \end{theorem}
    
    \begin{theorem}\label{thm:vcu-2eta}
    For any $\gamma>0$, there exists a hypothesis class $\H$ and adversary $\U$, 
    such that the sample complexity for $(\frac{3}{2}-\gamma,\epsilon,\delta)$-$\pac$ learn $\H$ is
    \begin{align*}
        m_l = \Omega\paren{\frac{\VCU(\H)}{\epsilon^2}+\frac{1}{\epsilon^2}\log\frac{1}{\delta}}
        \;,\;
        m_u = \infty.
    \end{align*}
    \end{theorem}
\paragraph{Open question.}
What is the optimal error rate in the agnostic setting when using only $\O\lr{\VCU}$ labeled examples?
\section{Learning with the 0-1 loss assuming robust realizability}\label{sec:improved-01-loss}
    In this section we learn with respect to the 0-1 loss, under robust realizability assumption.
    A Distribution $\D$ over $\X\times\Y$ is robustly realizable by $\H$ given a perturbation function $\U$, if there is $h\in\H$ such that not only $h$ classifies all points in $\D$ correctly, it also does so with respect to the robust loss function, that is,
    $\risk_{\U}(\H;\D)=0$. Note that our guarantees, only in this section, are with respect to the non-robust risk.
    The formal definition is in \cref{app:improved-01-loss}.
    A simple example for this model is the following. Let $\mathcal{H}$ be linear separators on $\mathcal{X}$ the unit ball in $\mathbb{R}^d$, and $\mathcal{U}$ as $\ell_2$ balls of radius $\gamma$, the robustly realizable distributions are separable with margin $\gamma$, where $\mathrm{VC}_{\mathcal{U}
     }(\mathcal{H}) = \frac{1}{\gamma^2}$ but $\mathrm{VC}(\mathcal{H}) = d+1$ can be arbitrarily larger.
    Moreover, we have the following example.
    (All proofs are in appendix \cref{app:improved-01-loss}.)
    \begin{proposition}\label{prop:vc2m-vcu1}
    For any $m\in \mathbb{N}$,
    there exist a hypothesis class $\H_m$ and distribution $\D$, such that $\D$ is robustly realizable by $\H_m$, $\VCU(\H_m)=1$, and $\VC(\H_m)= 2m$.
    \end{proposition}
    Standard $\VC$ theory does not ensure learning in this case. In this section we explain how we can learn in such a scenario with a small sample complexity (scales linearly in $\VCU$). Moreover, we show that it cannot be achieved via proper learners.
    
    \begin{theorem}\label{thm:improved-01}
    The sample complexity for learning a hypothesis class $\H$ with respect to the 0-1 loss, for any distribution $\D$ that is robustly realizable by $\H$, namely $\risk_{\U}\paren{\H;D}=0$,
    \begin{align*}
        \O\paren{\frac{\VCU(\H)}{\epsilon}\log^2\frac{\VCU(\H)}{\epsilon}+\frac{\log\frac{1}{\delta}}{\epsilon}}
        ,
        \Omega \paren{
        \frac{\VCU(\H)}{\epsilon}+\frac{\log\frac{1}{\delta}}{\epsilon}}.
    \end{align*}
    \end{theorem}
    This Theorem was an intermediate step in the proof of \cref{thm:realizable-sample-compexity}, and the sample complexity is the same as \cref{thm:partial-sample-complexity-realizable},
    $\O\paren{\Lambda_{\RE}(\epsilon,\delta,\H)}.$
    We show that there exists a robust ERM that fails in this setting (Proposition \ref{prop:erm-fails} in \cref{app:improved-01-loss}).
    Then, we claim that every proper learner fails.
    \begin{theorem}\label{thm:proper-fails}
    There exists $\H$ with $\VCU(\H)=1$, such that for any proper
    learning rule $\A :\paren{\X\times\Y}^* \rightarrow \H$, there exists a distribution $\D$ over $\X \times \Y$ that is robustly realizable by $\H$, i.e., $\risk_{\U}\paren{\H;\D}=0$,
    and it holds that
    $\risk\paren{\A(S);D}>\frac{1}{8}$ with probability at least $\frac{1}{7}$ over $S \sim \D^{m}$.
    \end{theorem}
\section*{Acknowledgments}
We are grateful to Omar Montasser for his helpful input, 
particularly inspiring steps 3 and 4 of 
the $\grass$ learning algorithm. We would like to thank Vinod Raman for his enlightening comments regarding the correctness of our algorithm. Finally, we thank the anonymous reviewers for their thoughtful comments, which helped us improve the presentation of our paper.

This project has received funding from the European Research Council (ERC) under the European Union’s Horizon 2020 research and innovation program (grant agreement No. 882396), by the Israel Science Foundation (grants 993/17, 1602/19), Tel Aviv University Center for AI and Data Science (TAD), and the Yandex Initiative for Machine Learning at Tel Aviv University.
I.A. is supported by the Vatat Scholarship from the Israeli Council for Higher Education and by Kreitman School of Advanced Graduate Studies.
\bibliography{paperbib}

\begin{thebibliography}{59}
\providecommand{\natexlab}[1]{#1}
\providecommand{\url}[1]{\texttt{#1}}
\expandafter\ifx\csname urlstyle\endcsname\relax
  \providecommand{\doi}[1]{doi: #1}\else
  \providecommand{\doi}{doi: \begingroup \urlstyle{rm}\Url}\fi

\bibitem[Alayrac et~al.(2019)Alayrac, Uesato, Huang, Fawzi, Stanforth, and
  Kohli]{alayrac2019labels}
Jean-Baptiste Alayrac, Jonathan Uesato, Po-Sen Huang, Alhussein Fawzi, Robert
  Stanforth, and Pushmeet Kohli.
\newblock Are labels required for improving adversarial robustness?
\newblock \emph{Advances in Neural Information Processing Systems}, 32, 2019.

\bibitem[Alon et~al.(2021)Alon, Hanneke, Holzman, and Moran]{alon2021theory}
Noga Alon, Steve Hanneke, Ron Holzman, and Shay Moran.
\newblock A theory of pac learnability of partial concept classes.
\newblock \emph{arXiv preprint arXiv:2107.08444}, 2021.

\bibitem[Ashtiani et~al.(2020)Ashtiani, Pathak, and Urner]{ashtiani2020black}
Hassan Ashtiani, Vinayak Pathak, and Ruth Urner.
\newblock Black-box certification and learning under adversarial perturbations.
\newblock In \emph{International Conference on Machine Learning}, pages
  388--398. PMLR, 2020.

\bibitem[Ashtiani et~al.(2022)Ashtiani, Pathak, and
  Urner]{ashtiani2022adversarially}
Hassan Ashtiani, Vinayak Pathak, and Ruth Urner.
\newblock Adversarially robust learning with tolerance.
\newblock \emph{arXiv preprint arXiv:2203.00849}, 2022.

\bibitem[Assouad(1983)]{assouad1983densite}
Patrick Assouad.
\newblock Densit{\'e} et dimension.
\newblock In \emph{Annales de l'institut Fourier}, volume~33, pages 233--282,
  1983.

\bibitem[Attias and Hanneke(2022)]{attias2022adversarially}
Idan Attias and Steve Hanneke.
\newblock Adversarially robust learning of real-valued functions.
\newblock \emph{arXiv preprint arXiv:2206.12977}, 2022.

\bibitem[Attias et~al.(2019)Attias, Kontorovich, and
  Mansour]{attias2019improved}
Idan Attias, Aryeh Kontorovich, and Yishay Mansour.
\newblock Improved generalization bounds for robust learning.
\newblock In \emph{Algorithmic Learning Theory}, pages 162--183. PMLR, 2019.

\bibitem[Auer and Ortner(2007)]{auer2007new}
Peter Auer and Ronald Ortner.
\newblock A new pac bound for intersection-closed concept classes.
\newblock \emph{Machine Learning}, 66\penalty0 (2):\penalty0 151--163, 2007.

\bibitem[Awasthi et~al.(2020)Awasthi, Frank, and Mohri]{awasthi2020adversarial}
Pranjal Awasthi, Natalie Frank, and Mehryar Mohri.
\newblock Adversarial learning guarantees for linear hypotheses and neural
  networks.
\newblock In \emph{International Conference on Machine Learning}, pages
  431--441. PMLR, 2020.

\bibitem[Awasthi et~al.(2021{\natexlab{a}})Awasthi, Frank, Mao, Mohri, and
  Zhong]{awasthi2021calibration}
Pranjal Awasthi, Natalie Frank, Anqi Mao, Mehryar Mohri, and Yutao Zhong.
\newblock Calibration and consistency of adversarial surrogate losses.
\newblock \emph{Advances in Neural Information Processing Systems}, 34,
  2021{\natexlab{a}}.

\bibitem[Awasthi et~al.(2021{\natexlab{b}})Awasthi, Frank, and
  Mohri]{awasthi2021existence}
Pranjal Awasthi, Natalie Frank, and Mehryar Mohri.
\newblock On the existence of the adversarial bayes classifier.
\newblock \emph{Advances in Neural Information Processing Systems}, 34,
  2021{\natexlab{b}}.

\bibitem[Balcan and Blum(2006)]{balcan200621}
Maria-Florina Balcan and Avrim Blum.
\newblock 21 an augmented pac model for semi-supervised learning.
\newblock 2006.

\bibitem[Balcan and Blum(2010)]{balcan2010discriminative}
Maria-Florina Balcan and Avrim Blum.
\newblock A discriminative model for semi-supervised learning.
\newblock \emph{Journal of the ACM (JACM)}, 57\penalty0 (3):\penalty0 1--46,
  2010.

\bibitem[Balcan and Long(2013)]{balcan2013active}
Maria-Florina Balcan and Phil Long.
\newblock Active and passive learning of linear separators under log-concave
  distributions.
\newblock In \emph{Conference on Learning Theory}, pages 288--316. PMLR, 2013.

\bibitem[Bhattacharjee et~al.(2021)Bhattacharjee, Jha, and
  Chaudhuri]{bhattacharjee2021sample}
Robi Bhattacharjee, Somesh Jha, and Kamalika Chaudhuri.
\newblock Sample complexity of robust linear classification on separated data.
\newblock In \emph{International Conference on Machine Learning}, pages
  884--893. PMLR, 2021.

\bibitem[Blum(2016)]{DBLP:reference/algo/Blum16}
Avrim Blum.
\newblock Semi-supervised learning.
\newblock In \emph{Encyclopedia of Algorithms}, pages 1936--1941. 2016.

\bibitem[Bshouty et~al.(2009)Bshouty, Li, and Long]{bshouty2009using}
Nader~H Bshouty, Yi~Li, and Philip~M Long.
\newblock Using the doubling dimension to analyze the generalization of
  learning algorithms.
\newblock \emph{Journal of Computer and System Sciences}, 75\penalty0
  (6):\penalty0 323--335, 2009.

\bibitem[Carmon et~al.(2019)Carmon, Raghunathan, Schmidt, Duchi, and
  Liang]{carmon2019unlabeled}
Yair Carmon, Aditi Raghunathan, Ludwig Schmidt, John~C Duchi, and Percy~S
  Liang.
\newblock Unlabeled data improves adversarial robustness.
\newblock \emph{Advances in Neural Information Processing Systems}, 32, 2019.

\bibitem[Chapelle et~al.(2009)Chapelle, Scholkopf, and Zien]{chapelle2009semi}
Olivier Chapelle, Bernhard Scholkopf, and Alexander Zien.
\newblock Semi-supervised learning (chapelle, o. et al., eds.; 2006)[book
  reviews].
\newblock \emph{IEEE Transactions on Neural Networks}, 20\penalty0
  (3):\penalty0 542--542, 2009.

\bibitem[Cullina et~al.(2018)Cullina, Bhagoji, and Mittal]{cullina2018pac}
Daniel Cullina, Arjun~Nitin Bhagoji, and Prateek Mittal.
\newblock Pac-learning in the presence of adversaries.
\newblock In \emph{Advances in Neural Information Processing Systems}, pages
  230--241, 2018.

\bibitem[Dan et~al.(2020)Dan, Wei, and Ravikumar]{dan2020sharp}
Chen Dan, Yuting Wei, and Pradeep Ravikumar.
\newblock Sharp statistical guaratees for adversarially robust gaussian
  classification.
\newblock In \emph{International Conference on Machine Learning}, pages
  2345--2355. PMLR, 2020.

\bibitem[Darnst{\"a}dt(2015)]{darnstadt2015optimal}
Malte Darnst{\"a}dt.
\newblock The optimal pac bound for intersection-closed concept classes.
\newblock \emph{Information Processing Letters}, 115\penalty0 (4):\penalty0
  458--461, 2015.

\bibitem[Darnst{\"a}dt et~al.(2013)Darnst{\"a}dt, Simon, and
  Sz{\"o}r{\'e}nyi]{darnstadt2013unlabeled}
Malte Darnst{\"a}dt, Hans~Ulrich Simon, and Bal{\'a}zs Sz{\"o}r{\'e}nyi.
\newblock Unlabeled data does provably help.
\newblock 2013.

\bibitem[David et~al.(2016)David, Moran, and Yehudayoff]{david2016supervised}
Ofir David, Shay Moran, and Amir Yehudayoff.
\newblock Supervised learning through the lens of compression.
\newblock \emph{Advances in Neural Information Processing Systems},
  29:\penalty0 2784--2792, 2016.

\bibitem[Feige et~al.(2015)Feige, Mansour, and Schapire]{feige2015learning}
Uriel Feige, Yishay Mansour, and Robert Schapire.
\newblock Learning and inference in the presence of corrupted inputs.
\newblock In \emph{Conference on Learning Theory}, pages 637--657, 2015.

\bibitem[Gin{\'e} and Koltchinskii(2006)]{gine2006concentration}
Evarist Gin{\'e} and Vladimir Koltchinskii.
\newblock Concentration inequalities and asymptotic results for ratio type
  empirical processes.
\newblock \emph{The Annals of Probability}, 34\penalty0 (3):\penalty0
  1143--1216, 2006.

\bibitem[G{\"o}pfert et~al.(2019)G{\"o}pfert, Ben-David, Bousquet, Gelly,
  Tolstikhin, and Urner]{gopfert2019can}
Christina G{\"o}pfert, Shai Ben-David, Olivier Bousquet, Sylvain Gelly, Ilya
  Tolstikhin, and Ruth Urner.
\newblock When can unlabeled data improve the learning rate?
\newblock In \emph{Conference on Learning Theory}, pages 1500--1518. PMLR,
  2019.

\bibitem[Graepel et~al.(2005)Graepel, Herbrich, and
  Shawe-Taylor]{graepel2005pac}
Thore Graepel, Ralf Herbrich, and John Shawe-Taylor.
\newblock Pac-bayesian compression bounds on the prediction error of learning
  algorithms for classification.
\newblock \emph{Machine Learning}, 59\penalty0 (1-2):\penalty0 55--76, 2005.

\bibitem[Hanneke(2009)]{hanneke2009theoretical}
Steve Hanneke.
\newblock \emph{Theoretical foundations of active learning}.
\newblock Carnegie Mellon University, 2009.

\bibitem[Hanneke(2016{\natexlab{a}})]{hanneke2016optimal}
Steve Hanneke.
\newblock The optimal sample complexity of pac learning.
\newblock \emph{The Journal of Machine Learning Research}, 17\penalty0
  (1):\penalty0 1319--1333, 2016{\natexlab{a}}.

\bibitem[Hanneke(2016{\natexlab{b}})]{hanneke2016refined}
Steve Hanneke.
\newblock Refined error bounds for several learning algorithms.
\newblock \emph{The Journal of Machine Learning Research}, 17\penalty0
  (1):\penalty0 4667--4721, 2016{\natexlab{b}}.

\bibitem[Hanneke et~al.(2019)Hanneke, Kontorovich, and
  Sadigurschi]{hanneke2019sample}
Steve Hanneke, Aryeh Kontorovich, and Menachem Sadigurschi.
\newblock Sample compression for real-valued learners.
\newblock In \emph{Algorithmic Learning Theory}, pages 466--488, 2019.

\bibitem[Haussler et~al.(1994)Haussler, Littlestone, and
  Warmuth]{haussler1994predicting}
David Haussler, Nick Littlestone, and Manfred~K Warmuth.
\newblock Predicting $\{$0, 1$\}$-functions on randomly drawn points.
\newblock \emph{Information and Computation}, 115\penalty0 (2):\penalty0
  248--292, 1994.

\bibitem[Khim and Loh(2018)]{khim2018adversarial}
Justin Khim and Po-Ling Loh.
\newblock Adversarial risk bounds via function transformation.
\newblock \emph{arXiv preprint arXiv:1810.09519}, 2018.

\bibitem[Kontorovich and Attias(2021)]{kontorovich2021fat}
Aryeh Kontorovich and Idan Attias.
\newblock Fat-shattering dimension of $ k $-fold maxima.
\newblock \emph{arXiv preprint arXiv:2110.04763}, 2021.

\bibitem[Levi et~al.(2021)Levi, Attias, and Kontorovich]{levi2021domain}
Matan Levi, Idan Attias, and Aryeh Kontorovich.
\newblock Domain invariant adversarial learning.
\newblock \emph{arXiv preprint arXiv:2104.00322}, 2021.

\bibitem[Long(2003)]{long2003upper}
Philip~M Long.
\newblock An upper bound on the sample complexity of pac-learning halfspaces
  with respect to the uniform distribution.
\newblock \emph{Information Processing Letters}, 87\penalty0 (5):\penalty0
  229--234, 2003.

\bibitem[Maurer and Pontil(2009)]{maurer2009empirical}
Andreas Maurer and Massimiliano Pontil.
\newblock Empirical bernstein bounds and sample variance penalization.
\newblock \emph{arXiv preprint arXiv:0907.3740}, 2009.

\bibitem[Mohri et~al.(2018)Mohri, Rostamizadeh, and
  Talwalkar]{mohri2018foundations}
Mehryar Mohri, Afshin Rostamizadeh, and Ameet Talwalkar.
\newblock \emph{Foundations of machine learning}.
\newblock MIT press, 2018.

\bibitem[Montasser et~al.(2019)Montasser, Hanneke, and Srebro]{montasser2019vc}
Omar Montasser, Steve Hanneke, and Nathan Srebro.
\newblock Vc classes are adversarially robustly learnable, but only improperly.
\newblock \emph{arXiv preprint arXiv:1902.04217}, 2019.

\bibitem[Montasser et~al.(2020{\natexlab{a}})Montasser, Goel, Diakonikolas, and
  Srebro]{montasser2020efficiently}
Omar Montasser, Surbhi Goel, Ilias Diakonikolas, and Nathan Srebro.
\newblock Efficiently learning adversarially robust halfspaces with noise.
\newblock In \emph{International Conference on Machine Learning}, pages
  7010--7021. PMLR, 2020{\natexlab{a}}.

\bibitem[Montasser et~al.(2020{\natexlab{b}})Montasser, Hanneke, and
  Srebro]{montasser2020reducing}
Omar Montasser, Steve Hanneke, and Nati Srebro.
\newblock Reducing adversarially robust learning to non-robust pac learning.
\newblock \emph{Advances in Neural Information Processing Systems},
  33:\penalty0 14626--14637, 2020{\natexlab{b}}.

\bibitem[Montasser et~al.(2021{\natexlab{a}})Montasser, Hanneke, and
  Srebro]{montasser2021adversarially}
Omar Montasser, Steve Hanneke, and Nathan Srebro.
\newblock Adversarially robust learning with unknown perturbation sets.
\newblock In \emph{Conference on Learning Theory}, pages 3452--3482. PMLR,
  2021{\natexlab{a}}.

\bibitem[Montasser et~al.(2021{\natexlab{b}})Montasser, Hanneke, and
  Srebro]{montasser2021transductive}
Omar Montasser, Steve Hanneke, and Nathan Srebro.
\newblock Transductive robust learning guarantees.
\newblock \emph{arXiv preprint arXiv:2110.10602}, 2021{\natexlab{b}}.

\bibitem[Moran and Yehudayoff(2016)]{moran2016sample}
Shay Moran and Amir Yehudayoff.
\newblock Sample compression schemes for vc classes.
\newblock \emph{Journal of the ACM (JACM)}, 63\penalty0 (3):\penalty0 1--10,
  2016.

\bibitem[Najafi et~al.(2019)Najafi, Maeda, Koyama, and
  Miyato]{najafi2019robustness}
Amir Najafi, Shin-ichi Maeda, Masanori Koyama, and Takeru Miyato.
\newblock Robustness to adversarial perturbations in learning from incomplete
  data.
\newblock \emph{Advances in Neural Information Processing Systems}, 32, 2019.

\bibitem[Sauer(1972)]{sauer1972density}
Norbert Sauer.
\newblock On the density of families of sets.
\newblock \emph{Journal of Combinatorial Theory, Series A}, 13\penalty0
  (1):\penalty0 145--147, 1972.

\bibitem[Schapire and Freund(2013)]{schapire2013boosting}
Robert~E Schapire and Yoav Freund.
\newblock Boosting: Foundations and algorithms.
\newblock \emph{Kybernetes}, 2013.

\bibitem[Schmidt et~al.(2018)Schmidt, Santurkar, Tsipras, Talwar, and
  Madry]{schmidt2018adversarially}
Ludwig Schmidt, Shibani Santurkar, Dimitris Tsipras, Kunal Talwar, and
  Aleksander Madry.
\newblock Adversarially robust generalization requires more data.
\newblock \emph{Advances in neural information processing systems}, 31, 2018.

\bibitem[Shalev-Shwartz and Ben-David(2014)]{shalev2014understanding}
Shai Shalev-Shwartz and Shai Ben-David.
\newblock \emph{Understanding machine learning: From theory to algorithms}.
\newblock Cambridge university press, 2014.

\bibitem[Uesato et~al.(2019)Uesato, Alayrac, Huang, Stanforth, Fawzi, and
  Kohli]{uesato2019labels}
Jonathan Uesato, Jean-Baptiste Alayrac, Po-Sen Huang, Robert Stanforth,
  Alhussein Fawzi, and Pushmeet Kohli.
\newblock Are labels required for improving adversarial robustness?
\newblock \emph{arXiv preprint arXiv:1905.13725}, 2019.

\bibitem[Urner et~al.(2011)Urner, Shalev-Shwartz, and
  Ben-David]{urner2011access}
Ruth Urner, Shai Shalev-Shwartz, and Shai Ben-David.
\newblock Access to unlabeled data can speed up prediction time.
\newblock In \emph{ICML}, 2011.

\bibitem[Vapnik and Chervonenkis(2015)]{vapnik2015uniform}
Vladimir~N Vapnik and A~Ya Chervonenkis.
\newblock On the uniform convergence of relative frequencies of events to their
  probabilities.
\newblock In \emph{Measures of complexity}, pages 11--30. Springer, 2015.

\bibitem[Warmuth(2004)]{warmuth2004optimal}
Manfred~K Warmuth.
\newblock The optimal pac algorithm.
\newblock In \emph{International Conference on Computational Learning Theory},
  pages 641--642. Springer, 2004.

\bibitem[Wei et~al.(2020)Wei, Shen, Chen, and Ma]{wei2020theoretical}
Colin Wei, Kendrick Shen, Yining Chen, and Tengyu Ma.
\newblock Theoretical analysis of self-training with deep networks on unlabeled
  data.
\newblock \emph{arXiv preprint arXiv:2010.03622}, 2020.

\bibitem[Xing et~al.(2021)Xing, Zhang, and Cheng]{xing2021adversarially}
Yue Xing, Ruizhi Zhang, and Guang Cheng.
\newblock Adversarially robust estimate and risk analysis in linear regression.
\newblock In \emph{International Conference on Artificial Intelligence and
  Statistics}, pages 514--522. PMLR, 2021.

\bibitem[Yin et~al.(2019)Yin, Kannan, and Bartlett]{yin2019rademacher}
Dong Yin, Ramchandran Kannan, and Peter Bartlett.
\newblock Rademacher complexity for adversarially robust generalization.
\newblock In \emph{International Conference on Machine Learning}, pages
  7085--7094. PMLR, 2019.

\bibitem[Zhai et~al.(2019)Zhai, Cai, He, Dan, He, Hopcroft, and
  Wang]{zhai2019adversarially}
Runtian Zhai, Tianle Cai, Di~He, Chen Dan, Kun He, John Hopcroft, and Liwei
  Wang.
\newblock Adversarially robust generalization just requires more unlabeled
  data.
\newblock \emph{arXiv preprint arXiv:1906.00555}, 2019.

\bibitem[Zhu and Goldberg(2009)]{zhu2009introduction}
Xiaojin Zhu and Andrew~B Goldberg.
\newblock Introduction to semi-supervised learning.
\newblock \emph{Synthesis lectures on artificial intelligence and machine
  learning}, 3\penalty0 (1):\penalty0 1--130, 2009.

\end{thebibliography}
\newpage
\appendix

\section{Additional preliminaries for \cref{sec:prelim}}\label{app:prelim}

\paragraph{Complexity measures.}
    The  capacity measures, $\VCU$, $\mathrm{RS}_\U$ and $\VC$, play an important role in our results. See Definitions \ref{def:vcu} and \ref{def:dimu} for the $\VCU$ and $\mathrm{RS}_\U$ dimensions. It holds that $\VCU(\H)\leq \mathrm{RS}_\U(\H) \leq\VC(\H)$, in Proposition \ref{prop:gap-vcu-dimu} we demonstrate an arbitrary gap between $\VCU$ and $\mathrm{RS}_\U$, the key parameters controlling the sample complexity of robust learnability.
    
    Denote the projection of a hypothesis class $\H$ on set $S=\sett{x_1,\ldots,x_k}$ by $\H|_{S}=\{(h(x_1),\ldots,$ $h(x_k)):h \in \H\}.$
    We say that a set $S \subseteq \X$ is shattered by $\H$ if $\sett{0,1}^S = \H|_{S}$, the $\VC$-dimension \citep*{vapnik2015uniform} of $\H$ is defined as the maximal size of a shattered set $S$.
    The dual hypothesis class $\H^*\subseteq\sett{0,1}^\H$ is defined as the set of all functions $f_x:\H\rightarrow\sett{0,1}$ where $f_x(h)=h(x).$ We denote the $\VC$-dimension of the dual class by $\VC^*(\H)$. It is known that $\VC^*(\H)<2^{\VC(\H)+1}$ \citep{assouad1983densite}.
     
    \begin{definition}[Sample compression scheme]
    A pair of functions
    $(\kappa,\rho)$ is a sample compression scheme of size $\ell$ for class $\H$ if for any $n\in\N$, $h\in \H$ and sample $S = \{(x_i,h(x_i))\}^n_{i=1}$, it holds for the compression function that $\kappa\left(S\right)\subseteq S$ and $|\kappa\left(S\right)|\leq \ell$, and the reconstruction function  $\rho\left(\kappa\left(S\right)\right)=\hat{h}$ satisfies $\hat{h}(x_i)= h(x_i)$  for any $i\in[n]$.
    \end{definition}
    \paragraph{Partial concept classes - \cite{alon2021theory}.}
    Let a partial concept class $\H\subseteq\sett{0,1,\star}^\X$. For $h\in\H$ and input $x$ such that $h(x)=\star$, we say that $h$ is undefined on $x$.
    The support of a partial hypothesis $h:\X\rightarrow \sett{0,1,\star}$ is the preimage of $\sett{0,1}$, formally, $h^{-1}(\sett{0,1})=\sett{x\in\X: h(x)\neq \star}$.
    The main motivation of introducing partial concepts classes, is that data-dependent assumptions can be modeled in a natural way that extends the classic theory of total concepts.
    
    The $\VC$-dimension of a partial class $\H$ is defined as the maximum size of a shattered set $S\subseteq \X$, where $S$ is shattered by $\H$ if the projection of $\H$ on $S$ contains all possible binary patterns, $\sett{0,1}^S\subseteq \H|_{S}$.
    The $\VC$-dimension also characterizes verbatim the $\pac$ learnability of partial concept classes. However, the uniform convergence argument does not hold, and the ERM principle does not ensure learning. The proof hinges on a combination of sample compression scheme and a variant of the \textit{one-Inclusion-Graph} algorithm \citep{haussler1994predicting}. In \cref{sec:realizable} we elaborate on the sample complexity of partial concept classes, and in \cref{app:algo-partial} we elaborate on the learning algorithms.
    The definitions of realizability and agnostic learning in the partial concepts sense generalizes the classic definitions for total concept classes. See \cite[Section 2 and Appendix C]{alon2021theory} for more details.

\section{Proofs for \cref{sec:knowing-support}}\label{app:knowing-support}
    \begin{proof}[of Proposition \ref{prop:gap-vcu-dimu}]
       We overview the construction by \citet{montasser2019vc}, which exemplifies an arbitrarily large gap between $\VCU$ and $\mathrm{RS}_\U$. 
    In this example $\VCU(\H) = 0$, $\mathrm{RS}_\U(\H) = \infty$, and $\VC(\H) = \infty$. 

    Define the Euclidean ball of radius $r$ perturbation function $\U(x)=B_r(x)$.
    Consider infinite sequences $(x_n)_{n\in\N}$ and $(z_n)_{n\in\N}$
    of points such that $\forall i\neq j,\; \U(x_i)\cap\U(x_j)=U(x_i)\cap\U(z_j)=\U(x_j)\cap\U(z_i)=\emptyset$, and $\forall i,\;\big| \U(x_i)\cap\U(z_i)\big|=1$.
    
    For a bit string $b\in \sett{0,1}^{\N}$, define a hypothesis $h_b: \sett{\U(x_i)\cup\U(z_i)}_{i\in \N}\rightarrow \sett{0,1}$ as follows. 
    $$
    h_b=
    \begin{cases}
    	h_b\Big(\U(x_i) \Big)=1 \;\wedge\; h_b\Big(\U(z_i)\setminus\U(x_i)\Big)=-1, & b_i=0\\
        h_b\Big(\U(z_i) \Big)=1 \;\wedge\; h_b\Big(\U(x_i)\setminus\U(z_i)\Big)=-1, & b_i=1.
     \end{cases}
    $$
    Define the hypothesis class $\H=\sett{h_b: b\in\sett{0,1}^{\N}}$. It holds that $\VCU(\H)=0$ and $\mathrm{RS}_\U=\infty$.
    \end{proof}

\section{Proofs for \cref{sec:realizable}}\label{app:realizable}

    Before proceeding to the proof, we present the following result on learning partial concept classes. Recall the definition of $\VC$ is in the context of partial concepts (see \cref{app:prelim}).
    \begin{theorem}[\cite{alon2021theory}, Theorem 34]\label{thm:partial-sample-complexity-realizable}
    Any partial concept class $\H$ with $\VC(\H)<\infty$ is $\pac$ learnable in the realizable setting with sample complexity,
    \begin{itemize}[leftmargin=0.7cm]
    \item $\Lambda_{\RE}\paren{\epsilon,\delta,\H} = 
    \O\paren{\min \left\{\frac{\VC(\H)}{\epsilon}\log\frac{1}{\delta},
    \frac{\VC(\H)}{\epsilon}\log^2\paren{\frac{\VC(\H)}{\epsilon}}+\frac{1}{\epsilon}\log\frac{1}{\delta}
    \right\}}$
    \item $\Lambda_{\RE}\paren{\epsilon,\delta,\H} = \Omega \paren{
    \frac{\VC(\H)}{\epsilon}+\frac{1}{\epsilon}\log\frac{1}{\delta}
    }$.
    \end{itemize}
    \end{theorem}

    \begin{proof}[of Theorem \ref{thm:realizable-sample-compexity}]
    At first,  we convert the hypothesis class $\H$ to $\H^{\star}_\U$ as described in Definition \ref{def:partial-robust-class}. Then, we employ the learning algorithm $\A$ for partial concepts on the partial concept class $\H^{\star}_\U$ and $S^l$, denote the resulting hypothesis by $h_1$. Note that we reduced the complexity of the class, since $\VC(\H^{\star}_\U)=\VCU(\H)$.
    \cref{thm:partial-sample-complexity-realizable} implies that whenever $m_l=|S^l|\geq \Tilde{\O}\paren{\frac{\VCU(\H)}{\epsilon}+\frac{1}{\epsilon}\log\frac{1}{\delta}}$, the hypothesis $h_1$ has a \underline{non-robust} error at most $\frac{\epsilon}{3}$ with probability $1-\frac{\delta}{2}$, with respect to the \underline{0-1 loss}. 
    Note that there exists $h\in \H$ that classifies correctly any point in $\D$ with respect to the robust loss function. So when we convert $\H$ to $\H^{\star}_\U$, the "partial version" of $h$ still classifies correctly any point in $S^l$, and does not return any $\star$, which always counts as a mistake. 
    Algorithm $\A$ guarantees to return a hypothesis that is $\epsilon$-optimal with respect to the $0$-$1$ loss, with high probability.
    Observe that after these two steps, we obtain the following intermediate result.
    Whenever a distribution $\D$ is robustly realizable by a hypothesis class $\H$, i.e., $\risk_{\U}(\H;\D)=0$, we have an algorithm that learns this class with respect to the \underline{0-1 loss}, with sample complexity of
    \begin{align}\label{eq:realizable-robust}
         \Upsilon(\epsilon,\delta,\H,\U)
         =
         \O\paren{\Lambda_{\RE}(\epsilon,\delta,\H) } =
         \O\paren{\frac{\VCU(\H)}{\epsilon}\log^2\frac{\VCU(\H)}{\epsilon}+\frac{1}{\epsilon}\log\frac{1}{\delta}}.
    \end{align}
    The sample complexity of this model is defined formally in Definition \ref{def:robust-realizable-pac}. in \cref{sec:improved-01-loss} present more results for this model.
    
    In the third step, we label an independent unlabeled sample $S^{u}_{\X}\sim \D_{\X}^{m_u}$ with $h_1$, denote this labeled sample by $S^u$. Define a distribution $\Tilde{\D}$ over $\X\times\Y$ by
    $$\Tilde{\D}(x,h_1(x))=\D_{\X}(x),$$
    and so $S^u$ is an i.i.d. sample from $\Tilde{\D}$.
    We argue that the robust error of $\H$ with respect to $\Tilde{\D}$ is at most $\frac{\epsilon}{3}$, i.e., $\risk_{\U}(\H;\Tilde{\D})\leq\frac{\epsilon}{3}$.
    Indeed, we show that $h_{\text{opt}}\in\argmin_{h\in\H}\risk_{\U}(h;\D)$ has a robust error of at most $\frac{\epsilon}{3}$ on $\Tilde{\D}$.
    Note that,
    \begin{align}\label{ineq:risk-D-Dtilde}
    \risk_{\U}(\H;\Tilde{\D})
    \leq
    \E_{(x,y)\sim \D}\sqparen{\ell_{\U}(h_{\text{opt}};x,h_1(x))}
    =
    \E_{(x,y)\sim \Tilde{\D}}\sqparen{\ell_{\U}(h_{\text{opt}};x,y)}. 
    \end{align}
    Observe that the following holds for any $(x,y)$,
    \begin{align}\label{ineq:robust-01}
    \ell_{\U}(h_{\text{opt}};x,h_1(x))\leq \ell_{\U}(h_{\text{opt}};x,y)+\ell_{0\text{-}1}(h_1;x,y).
    \end{align}

    Indeed, the right hand side is 0, whenever $h_1$ classifies $(x,y)$ correctly, and $h_{\text{opt}}$ robustly classifies $(x,y)$ correctly, which implies that the left hand side is 0 as well.
    
    By taking the expectation on \cref{ineq:robust-01} we have,
    \begin{align}\label{ineq:robust-01-expec}
    \E_{(x,y)\sim \D}[\ell_{\U}(h_{\text{opt}};x,h_1(x))]\leq \E_{(x,y)\sim \D}[\ell_{\U}(h_{\text{opt}};x,y)]+\E_{(x,y)\sim \D}[\ell_{0\text{-}1}(h_1;x,y)].    
    \end{align}
    
    Combining it together, we obtain
    \begin{align*}
    \risk_{\U}(\H;\Tilde{\D})
    &\leq
    \E_{(x,y)\sim \Tilde{\D}}\sqparen{\ell_{\U}(h_{\text{opt}};x,y)}
    \\
    &\overset{(\text{i})}{=}
    \E_{(x,y)\sim \D}[\ell_{\U}(h_{\text{opt}};x,h_1(x))]
    \\
    &\overset{(\text{ii})}{\leq}
    \E_{(x,y)\sim \D}[\ell_{\U}(h_{\text{opt}};x,y)]+\E_{(x,y)\sim \D}[\ell_{0\text{-}1}(h_1;x,y)]
    \\
    &\leq
    \frac{\epsilon}{3}
    \end{align*}
    where (i) follows from \cref{ineq:risk-D-Dtilde} and (ii) follows from \cref{ineq:robust-01-expec}.
    
    Finally, we employ an agnostic adversarially robust \underline{supervised} $\pac$ learner $\B$ for the class $\H$ on $S^u\sim \Tilde{\D}^{m_u}$, that should be of size of the sample complexity of agnostically robust learn $\H$ with respect to $\U$, when the optimal robust error of hypothesis from $\H$ on $\Tilde{\D}$ is at most $\frac{\epsilon}{3}$. 
    We are guaranteed that the resulting hypothesis $h_2$ has a \underline{robust} error of at most $\frac{\epsilon}{3}+\frac{\epsilon}{3} =\frac{2\epsilon}{3}$ on $\Tilde{\D}$, with probability $1-\frac{\delta}{2}$. 
    We observe that the total variation distance between $\D$ and $\Tilde{\D}$ is at most $\frac{\epsilon}{3}$, and as a result, $h_2$ has a robust error of at most $\frac{2\epsilon}{3}+\frac{\epsilon}{3} = \epsilon$ on $\D$, with probability $1-\delta$.
    
    We conclude that a size of $|S^{u}_{\X}|=m_u =\Lambda_{\AG}\paren{1,\frac{\epsilon}{3},\frac{\delta}{2},\H,\U,\eta=\frac{\epsilon}{3}}$ unlabeled samples suffices, in addition to $m_l= \Tilde{\O}\paren{\frac{\VCU(\H)}{\epsilon}+\frac{1}{\epsilon}\log\frac{1}{\delta}}$ labeled samples which are required in the first 2 steps.
    \end{proof}
    We now prove \cref{thm:realizable-sample-compexity-cor}.
    The following data-dependent compression based generalization bound is a variation of the classic bound by \citet{graepel2005pac}. It follows the same arguments while using the empirical Bernstein bound instead of Hoeffding's inequality. A variation of this bound, with respect to the $0$-$1$ loss, appears in \cite[Lemma 42]{alon2021theory}, and \cite[Section 5]{maurer2009empirical}. 
    The exact same arguments follows for the robust loss as well.

    This bound includes the empirical error factor, and as soon as we call the compression based learner on a sample that is "nearly" realizable (Step 4 in the algorithm), we can improve the sample complexity of the agnostic robust supervised learner, such that the dependence on $\epsilon^2$ is reduced to $\epsilon$, for the unlabeled sample size.
    \begin{lemma}[Agnostic sample compression generalization bound]\label{lem:sample-compression-bernstein}
    For any sample compression scheme $(\kappa,\rho)$, for any $m \in \NN$ and $\delta\in (0,1)$, for any distribution $\D$ over $\X\times\sett{0,1}$, for $S\sim \D^m$, with probability $1-\delta$,
    {\small
    \begin{align*}
    \left|{\risk}_{\U}\paren{\rho({\kappa({S})});\D}-\widehat{\risk}_{\U}\paren{\rho({\kappa({S})});S}\right|
    \leq
    \O\paren{\sqrt{\widehat{\risk}_{\U}\paren{\rho({\kappa({S})});S}\frac{\paren{|\kappa(S)|\log(m)+\log \frac{1}{\delta}}}{m}} + \frac{|\kappa(S)|\log(m)+\log \frac{1}{\delta}}{m}}.
    \end{align*}
    }%
    \end{lemma}
    
    \begin{proof}[of \cref{thm:realizable-sample-compexity-cor}]
    \citet[Theorem 6]{montasser2019vc} introduced an agnostic robust supervised learner 
    that requires the following \underline{labeled} sample size,
    $$\Lambda_{\AG}\paren{1,\epsilon,\delta,\H,\U,\eta} = \Tilde{\O}\paren{\frac{\VC(\H)\VC^*(\H)}{\epsilon^2}+\frac{\log \frac{1}{\delta}}{\epsilon^2}}.$$
    
    Their argument for generalization is based on classic compression generalization bound by \citet{graepel2005pac}, adapted to the robust loss. See \citet[Lemma 11]{montasser2019vc}.
    
    We show that in our use case we can deduce a stronger bound. We employ the agnostic learner on a distribution which is "close" to realizable, the error of the optimal $h\in\H$ is at most $\eta = \frac{\epsilon}{3}$,
    and so we need $\Lambda_{\AG}\paren{1,\frac{\epsilon}{3},\frac{\delta}{2},\H,\U,\eta=\frac{\epsilon}{3}}$ unlabeled examples.
    As a result,
    we obtain an improved bound by using a data-dependant generalization bound described in Lemma \ref{lem:sample-compression-bernstein}.
    
    This improves the unlabeled sample size (denoted by $m_u$), and reduces its dependence on $\epsilon^2$ to $\epsilon$.
     Overall we obtain a sample complexity of
    \begin{align*}
            m_u = \Tilde{\O}\paren{\frac{\VC(\H)\VC^*(\H)}{\epsilon}+\frac{\log\frac{1}{\delta}}{\epsilon}}, \; \; m_l= \O\paren{\frac{\VCU(\H)}{\epsilon}\log^2\frac{\VCU(\H)}{\epsilon}+\frac{\log\frac{1}{\delta}}{\epsilon}}.
    \end{align*}
    \end{proof}
    \begin{proof}[of \cref{thm:improper}]
    This proof is identical to \citep[Lemma 3]{montasser2019vc}, We overview the idea of the proof.
    If the proof is true for a labeled sample, it remains true when some of the labels are missing. 
    
    Define the following hypothesis class $\H_m \subseteq [0,1]^\X$. Define the instance space $\X=\sett{x_1,\ldots,x_m}\subseteq \mathbb{R}$ and a perturbation function $\U:\X\rightarrow 2^\X$, such that the perturbation sets of the instances do not intersect, that is, $ \forall i,j\in[m]:\; \U\paren{x_i}\cap \U\paren{x_j}$. We can simply take the perturbations sets to be $\ell_2$ unit balls, $\U(x)=\sett{z\in\mathbb{R}: \norm{z-x}_2\leq 1}$ such that $ \forall i,j\in[m]:\; \norm{x_i-x_j}_2 >2$. Now, each $h_b\in\H_m$ is represented by a bit string $b=\lrset{0,1}^m$, such that if $b_i=1$, then there exist an adversarial example in $\U\paren{x_i}$ that is unique for each $h_b$, and otherwise, the function is consistent on $\U(x_i)$. 

Formally, for each $i\in[m]$ define a bijection $\psi_i:x_i\times \H_m \rightarrow \U\paren{x_i}\setminus{\sett{{x_i}}}$. Define $\H_m = \sett{h_b:b\in\sett{0,1}^m}$, such that for any 
$x_i\in\X$, $h_b$ is defined by
    $$
    h_b(x_i)=
    \begin{cases}
        h_b\LR{\U\lr{x_i}\setminus{\psi_i\lr{x_i,h_b}}}=0 \;\wedge\; h_b\LR{\psi_i\lr{x_i,h_b}}=1, & b_i=1,\\
        h_b\LR{\U\lr{x_i}}=0, & b_i=0.
     \end{cases}
    $$
Note that since $\psi_i$ is a bijection, and different functions with $b_i=1$ have a different perturbation for $x_i$ that causes a misclassification.

For a function class $\H$, define the robust loss class $\calL^\U_\H= \sett{\paren{x,y}\mapsto \sup_{z\in\U\paren{x}}\mathbb{I}\sett{h(z)\neq y}: h\in \H}$. It holds that $\VC\paren{\H_m}\leq 1$ and $\VC\paren{\calL^\U_{\H_m}}=m$ (see \cite[Lemma 2]{montasser2019vc}).

We define a function class $\Tilde{\H}_{3m}=\sett{h_b\in\H_{3m}: \sum^{3m}_{i=1}b_i=m}$. In words, we are keeping only functions in $\H_{3m}$ that are robustly correct on exactly $2m$ points. Note that the function $h_{\vec{0}}$
(bit string of all zeros) which is robustly correct on all $3m$ points, is not the class. 

The idea is that we can construct a family of $\binom{3m}{2m}$ distributions, such that each distribution is supported on $2m$ points from $\X=\sett{x_1,\ldots,x_{3m}}$. Now, if we have a proper learning rule, observing only $m$ points, the algorithm has no information which are the remaining $m$ points in the support (out of $2m$ possible points in $\X$). 
For each such a distribution there exists $h\in\Tilde{\H}_{3m}$, with zero robust error. We can follow a standard proof of the no-free-lunch theorem \citep[e.g.,][Section 5]{shalev2014understanding}, showing via the probabilistic method, that there exists a distribution on which the algorithm has constant error, although there is an optimal function in $\Tilde{\H}_{3m}$.
 See \cite[Lemma 3]{montasser2019vc} for the full proof.
\end{proof}
    \section{Proofs for \cref{subsec:agnostic}}\label{app:agnostic}
    Before proceeding to the proof, we present the following result on agnostically learning partial concept classes.
    Recall the definition of $\VC$ is in the context of partial concepts (see \cref{app:prelim}).
    
    \begin{theorem}[\cite{alon2021theory}, Theorem 41]\label{thm:partial-sample-complexity-agnostic}
    Any partial concept class $\H$ with $\VC(\H)<\infty$ is agnostically $\pac$ learnable with sample complexity,
    
        \begin{itemize}[leftmargin=0.7cm]
            \item $\Lambda_{\AG}\paren{\epsilon,\delta,\H} = 
            \O\paren{\ \frac{\VC(\H)}{\epsilon^2}\log^2\paren{\frac{\VC(\H)}{\epsilon^2}}+\frac{1}{\epsilon^2}\log\frac{1}{\delta}}$.
            \item $\Lambda_{\AG}\paren{\epsilon,\delta,\H} = \Omega \paren{
            \frac{\VC(\H)}{\epsilon^2}+\frac{1}{\epsilon^2}\log\frac{1}{\delta}
            }$.
        \end{itemize}
    \end{theorem}

    \begin{proof}[of Theorem \ref{thm:agnostic-sample-compexity}]
    We follow the same steps as in the proof of the realizable case, with the following difference. In the first two steps of the algorithm we learn with respect to the $0$-$1$ loss, with error of $\eta$ (the optimal robust error of a hypothesis in $\H$) and not $0$, which leads eventually to approximation of $\,3\eta\,$ for learning with the robust loss.
    
    At first, we convert the class $\H$ into $\H_{\U}^{\star}$, on which we employ the learning algorithm $\A$ for partial concepts with with the sample $S^l$.
    \cref{thm:partial-sample-complexity-agnostic} implies that whenever $m_l=|S^l|\geq \Tilde{\O}\paren{\frac{\VCU(\H)}{\epsilon^2}+\frac{1}{\epsilon^2}\log\frac{1}{\delta}}$, the resulting hypothesis $h_1$ returned by algorithm $\A$ has a \underline{non-robust} error at most $\eta+\frac{\epsilon}{3}$ with probability $1-\frac{\delta}{2}$, with respect to the \underline{0-1 loss}, where $\eta=\risk_{\U}\paren{\H;\D}$. 
    Note that there exists $h\in \H$ with robust error of $\eta$ on $\D$. The "partial version" of $h$ has an error of $\eta$ on $\D$ with respect to the $0$-$1$ loss. As a result,
    algorithm $\A$ guarantees to return a hypothesis that is $\epsilon$-optimal with respect to the $0$-$1$ loss, with high probability.
    
    We label an independent unlabeled sample $S^{u}_{\X}\sim \D_{\X}^{m_u}$ with $h_1$, denote this labeled sample by $S^u$. Similarly to the realizable case, define a distribution $\Tilde{\D}$ over $\X\times\Y$ by
    $$\Tilde{\D}(x,h_1(x))=\D_{\X}(x),$$
    and so $S^u$ is an i.i.d. sample from $\Tilde{\D}$.
     We argue that the robust error of $\H$ with respect to $\Tilde{\D}$ is at most $2\eta+\frac{\epsilon}{3}$, i.e., $\risk_{\U}(\H;\Tilde{\D})=2\eta +\frac{\epsilon}{3}$,
    by showing that $h_{\text{opt}}=\argmin_{h\in\H}\risk_{\U}(h;\D)$ has a robust error of at most $2\eta+\frac{\epsilon}{3}$ on $\Tilde{\D}$.

    \cref{ineq:risk-D-Dtilde,ineq:robust-01,ineq:robust-01-expec} still hold as in the realizable case proof. Combining it together, we have
    \begin{align*}
    \risk_{\U}(\H;\Tilde{\D})
    &\leq
    \E_{(x,y)\sim \Tilde{\D}}\sqparen{\ell_{\U}(h_{\text{opt}};x,y)}
    \\
    &\overset{(\text{i})}{=}
    \E_{(x,y)\sim \D}[\ell_{\U}(h_{\text{opt}};x,h_1(x))]
    \\
    &\overset{(\text{ii})}{\leq}
    \E_{(x,y)\sim \D}[\ell_{\U}(h_{\text{opt}};x,y)]+\E_{(x,y)\sim \D}[\ell_{0\text{-}1}(h_1;x,y)]
    \\
    &\leq
    \eta+\eta+\frac{\epsilon}{3}
    \\
    &=
    2\eta+\frac{\epsilon}{3},
    \end{align*}
    where (i) follows from \cref{ineq:risk-D-Dtilde} and (ii) follows from \cref{ineq:robust-01-expec}.
    
    Finally, we employ an agnostic adversarially robust \underline{supervised} $\pac$ learner $\B$ for the class $\H$ on $S^u\sim \Tilde{\D}^{m_u}$, that should be of size of the sample complexity of agnostically robust learn $\H$ with respect to $\U$, when the optimal robust error of hypothesis from $\H$ on $\Tilde{\D}$ is at most $2\eta +\frac{\epsilon}{3}$. 
    We are guaranteed that the resulting hypothesis $h_2$ has a \underline{robust} error of at most $2\eta + \frac{\epsilon}{3}+\frac{\epsilon}{3} =2\eta+\frac{2\epsilon}{3}$ on $\Tilde{\D}$, with probability $1-\frac{\delta}{2}$. 
    We observe that the total variation distance between $\D$ and $\Tilde{\D}$ is at most $\eta + \frac{\epsilon}{3}$, and as a result, $h_2$ has a robust error of at most $2\eta+\frac{2\epsilon}{3} + \eta + \frac{\epsilon}{3}= 3\eta + \epsilon$ on $\D$, with probability $1-\delta$.
    
    We conclude that a size of $|S^{u}_{\X}|=m_u =\Lambda_{\AG}\paren{1,\frac{\epsilon}{3},\frac{\delta}{2},\H,\U,2\eta+\frac{\epsilon}{3}}$ unlabeled sample suffices, in addition to the $m_l= 
    \O\paren{\frac{\VCU(\H)}{\epsilon^2}\log^2\frac{\VCU(\H)}{\epsilon^2}+\frac{\log\frac{1}{\delta}}{\epsilon^2}}$
    labeled samples which are required in the first 2 steps.
    We remark that the best known value of $\Lambda_{\AG}\paren{1,\epsilon,\delta,\H,\U,\eta}$ is $\Tilde{\O}\paren{\frac{\VC(\H)\VC^*(\H)}{\epsilon^2}+\frac{\log\frac{1}{\delta}}{\epsilon^2}}$.
    \end{proof}
\begin{proof}[of \cref{thm:dimu-labels-agnostic}]
We give a proof sketch,
this is similar to \cite[Theorem 10]{montasser2019vc}, knowing the marginal distribution $\D_{\X}$ does not give more power to the learner. The argument is based on the standard lower bound for $\VC$ classes (for example \cite[Section 3]{mohri2018foundations}). Let $S=\{x_1,\ldots,x_k\}$ be a maximal set that is $\U$-robustly shattered by $\H$. 

Let $z^+_1,z^-_1,\ldots,z^+_k,z^-_k$
be as in Definition \ref{def:dimu}, and note that for $i\neq j$, $z^+_i\neq z^+_j$ and $z^-_i\neq z^-_j$. Define a distribution $\D_{\boldsymbol{\sigma}}$ for any possible labeling $\boldsymbol{\sigma}=(\sigma_1,\ldots,\sigma_k)\in\{0,1\}^k$ of $S$.

$$
\forall j\in[k]:\,
\begin{cases}
  \D_{\boldsymbol{\sigma}}(z^+_j,1)=\frac{1-\alpha}{2k} \;\wedge\; \D_{\boldsymbol{\sigma}}(z^-_j,0)=\frac{1+\alpha}{2k}  & \sigma_j = 0, \\
  \D_{\boldsymbol{\sigma}}(z^+_j,1)=\frac{1+\alpha}{2k} \;\wedge\; \D_{\boldsymbol{\sigma}}(z^-_j,0)=\frac{1-\alpha}{2k}  & \sigma_j = 1.
\end{cases}
$$

We can now choose $\alpha$ as a function of $\epsilon,\delta$ in order to get a lower bound on the sample complexity $|S|\gtrsim \frac{\mathrm{RS}_\U}{\epsilon^2}$.
\end{proof}

\begin{proof}[of \cref{thm:vcu-2eta}]
We take the construction in Proposition \ref{prop:gap-vcu-dimu}, where there is an arbitrary gap between $\VCU$ and $\mathrm{RS}_\U$. 

Recall that on every pair $(x,z)$ in Proposition \ref{prop:gap-vcu-dimu} the optimal error is $\eta=1/2$. On such unlabeled pairs, the learner can only randomly choose a prediction, and the error is $3/4$.
We have $\VCU=0$, and the labeled sample size is $\frac{1}{\epsilon^2}\log\frac{1}{\delta}$. 
As $(\mathrm{RS}_\U-\frac{1}{\epsilon^2}\log\frac{1}{\delta})$ grows, the gap between the learner and the optimal classifier is approaching $3/2$, which means that for any $\gamma>0$ we can pick $\mathrm{RS}_\U$ such that error of $(\frac{2}{3}-\gamma)\eta$ is not possible. 

In order to prove the case of any $0<\eta\leq 1/2$, we can just add points such that their perturbation set does not intersect with any other perturbation set, and follow the same argument.
\end{proof}
\section{Auxiliary definitions and proofs for \cref{sec:improved-01-loss}}\label{app:improved-01-loss}
Definition of the model.
    \begin{definition}[(non-robust) $\pac$ learnability for robustly realizable distributions]\label{def:robust-realizable-pac}
     For any $\epsilon,\delta\in (0,1)$, the sample complexity of $(\epsilon,\delta)$-$\pac$ learning for a class $\H$, denoted by $\Upsilon(\epsilon,\delta,\H,\U)$, is the smallest integer $m$ for which there exists a learning algorithm $\A :\paren{\X\times\Y}^*\rightarrow \Y^{\X}$, such that for every distribution $\D$ over $\X\times\Y$ robustly realizable by $\H$ with respect to a perturbation function $\U:\X\rightarrow 2^{\X}$, namely $\risk_{\U}\paren{\H;D}=0$, for a random sample $S \sim \D^m$, it holds that 
    $$\Pr\paren{\risk\paren{\A(S);D}\leq \epsilon}>1-\delta.$$
    If no such $m$ exists, define $\Upsilon(\epsilon,\delta,\H,\U) = \infty$,
    and $\H$ is not $(\epsilon,\delta)$-$\pac$ 
    for distributions that are robustly realizable by $\H$ with respect to $\U$. 
    \end{definition}
\begin{proof}[of Proposition \ref{prop:vc2m-vcu1}]
    Define the uniform distribution $\D$ over the support $\sett{(x_1,1),\ldots,(x_{2m},1)}$, such that $\bigcap_{i=1}^{2m} \U(x_i)\neq\emptyset$.
    Define $\H:\X\rightarrow 2^\X$ to be all binary functions over $\X$. Note that the $\D$ is robustly realizable by $\H$, the constant function that return always $1$ has no error. Moreover we have $\VCU=1$, and $\VC= 2m$, for any $m\in\N$.
\end{proof}
\begin{proof}[of \cref{thm:improved-01}]
    We follow only the first two steps of the generic \cref{alg:generic-algo}. Namely, take a labeled sample $S$ and a hypothesis class $\H$ and create the partial hypothesis class $\H^{\star}_\U$. Assuming that the distribution is robustly realizable by $\H$, we end up in a realizable setting of learning a partial concept class $\H^{\star}_\U$. 
    
    In the second step of the algorithm we call a learning algorithm for partial concept classes (\cref{para:one-inclu}) in order to do so. 
    The sample complexity is the same as \cref{thm:partial-sample-complexity-realizable}, $\Upsilon(\epsilon,\delta,\H,\U)=
    \O\paren{\Lambda_{\RE}(\epsilon,\delta,\H)}.$
    Has we have shown in the proof of \cref{thm:realizable-sample-compexity}, \cref{eq:realizable-robust}, this implies the Theorem.
\end{proof}
\begin{proposition}\label{prop:erm-fails}
    Consider the distribution $\D$ and the hypothesis class $\H$ in \cref{prop:vc2m-vcu1}. There exists a robust ERM algorithm returning a hypothesis $h_{\text{ERM}}\in \H$,
    such that $\risk\paren{h_{\text{ERM}};\D}\geq \frac{1}{4}$ with probability $1$ over $S \sim \D^{m}$.
    \end{proposition}
\begin{proof}
    Consider the following robust ERM. For any sample of size $m$, return $1$  on the sample points and randomly choose a label for out of sample points. The error rate of such a robust ERM is at least $1/4$ with probability $1$.
\end{proof}
\begin{proof}[of \cref{thm:proper-fails}]
This follows from a similar no-free-lunch argument for $\VC$ classes \citep[e.g.,][Section 5]{shalev2014understanding}. We briefly explain the proof idea. 

Take the distribution $\D$, and the class $\H$ from Proposition \ref{prop:erm-fails} with $\VCU(\H)=1$ and $\VC(\H)=3m$. Keep functions that are robustly self consistent only on $2m$ points. 
Construct all of distributions on $2m$ points from the support of $\D$. We have $\binom{3m}{2m}$ such distributions, and on each one of them is robustly realizable by different $h\in\H$.
The idea is the that a proper leaner observing only $m$ points should guess which are the remaining $m$ points the support of the distribution. 
There rest of the proof follows from the no-free-lunch proof.
It can be shown formally via the probabilistic method, that for every proper rule, there exist a distribution on which the error is constant with fixed probability.
\end{proof}
\section{Learning algorithms for partial concept classes}\label{app:algo-partial}
Here we overview the algorithmic techniques from \citet[Theorem 34 and 41]{alon2021theory}, for learning partial concepts in the realizable and agnostic settings.
We use these algorithms in step 2 of our \cref{alg:generic-algo}.

\paragraph{One-inclusion graph algorithm for partial concept
classes.}\label{para:one-inclu}
        We briefly discuss the algorithm, for the full picture, see \cite*{warmuth2004optimal,haussler1994predicting}.
        The one-inclusion algorithm for a class $\F\subseteq\sett{0,1,\star}^\X$ gets an input of unlabeled examples $S = (x_1,\ldots,x_m)$ and labels $(y_1,\ldots,y_{i-1},y_{i+1},\ldots,y_m)$ that are consistent with some $f\in \F$, that is, $f(x_k)=y_k$ for all $k\neq i$.
        It guarantees an $(\epsilon,\delta)$- $\pac$ learner in the realizable setting, with sample complexity of $\Lambda_{\RE}\paren{\epsilon,\delta,\H} = 
        \O\paren{ \frac{\VC(\H)}{\epsilon}\log\frac{1}{\delta}}$ 
        as mentioned in \cref{thm:partial-sample-complexity-realizable}.
        
        Here is a description of the algorithm.
        At first, construct the one-inclusion graph. For any $j\in[m]$ and $f\in \F|_S$ define $E_{j,f} = \sett{f'\in \F|_S: f'(x_k)=f(x_k), \forall k\neq j}$, that is, all functions in $\F|_S$ that are consistent with $f$ on $S$, except the point $x_j$. Define the set of edges 
        $E=\sett{E_{j,f}: j\in[m],f\in \F|_S}$,
        and the set vertices $V=\F|_S$ of the one-inclusion graph $G=(V,E)$. An orientation function $\psi: E \rightarrow V$ for an undirected graph $G$ is an assignment of a direction to each edge, turning $G$ into a directed graph. Find an orientation $\psi$ that minimizes the out-degree of $G$. For prediction of $x_i$, pick $f\in V$ such that $f(x_k)=y_k$ for all $k\neq i$, and output $\psi(E_{i,f})(x_i)$.
        
        Note that this algorithm is transductive, in a sense that in order to predict the label of a test point, it uses the entire training sample to computes its prediction.
\paragraph{Boosting and compression schemes.}   
Recall the well known boosting algorithm, $\alpha$-Boost \citep[pages 162-163]{schapire2013boosting}, which is a simplified version of AdaBoost, where the returned function is a simple majority over weak learners, instead of a weighted majority. For a hypothesis class $\H$ and a sample of size $m$, the algorithm yields a compression scheme of size $\O\paren{\VC(\H)\log(m)}$. Recall the following generalization bound based on sample compression scheme.
\begin{lemma}[\cite{graepel2005pac}]\label{lem:compression-generalization}
Let a sample compression scheme $(\kappa,\rho)$, and a loss function $\ell:\mathbb{R}\times\mathbb{R}\rightarrow [0,1]$. In the agnostic case, for any $\kappa(S)\lesssim m$, any $\delta\in (0,1)$, and any distribution $\D$ over $\X\times\lrset{0,1}$, for $S\sim \D^m$, with probability $1-\delta$, 

\begin{align*}
\left|{\risk}\paren{\rho({\kappa({S})});\D}-\widehat{\risk}\paren{\rho({\kappa({S})});S}\right|
\leq
\O\paren{\sqrt{\frac{\paren{|\kappa(S)|\log(m)+\log \frac{1}{\delta}}}{m}}}.
\end{align*}
\end{lemma}
The learning algorithm for the realizable setting is $\alpha$-Boost, where the weak learners are taken from the one-inclusion graph algorithm. As mentioned in \cref{thm:partial-sample-complexity-realizable},
this obtains an upper bound of $\Lambda_{\RE}\paren{\epsilon,\delta,\H} = 
    \O\paren{
    \frac{\VC(\H)}{\epsilon}\log^2\paren{\frac{\VC(\H)}{\epsilon}}+\frac{1}{\epsilon}\log\frac{1}{\delta}}$.
    
For the agnostic setting, follow a reduction to the realizable case suggested by \citet{david2016supervised}. The reduction requires a construction of a compression scheme based on Boosting algorithm. Roughly speaking, the reductions works as follows. 
Denote $\Lambda_{\RE} = \Lambda_{\RE}(1/3,1/3,\H)$, the sample complexity of $(1/3,1/3)$-$\pac$ learn $\H$, in the realizable case.
Now, $\Lambda_{\RE}$ samples suffice for weak learning for any distribution $\D$ on a given sample $S$. 

Find the maximal subset $S'\subseteq S$
such that $\inf_{h\in\H}\widehat{\risk}\paren{h;S'}=0$. Now, $\Lambda_{\RE}$ samples suffice for weak robust learning for any distribution $\D$ on $S'$.
Execute the $\alpha$-boost algorithm on $S'$, with parameters $\alpha = \frac{1}{3}$ and number of boosting rounds $T=\O\paren{\log\paren{|S'|}}$, where each weak learner is trained on $\Lambda_{\RE}$ samples. The returned hypothesis
$\bar{h}=\maj\paren{\hat{h}_1,\ldots,\hat{h}_T}$
satisfies that  $\widehat{\risk}\paren{\bar{h};S'}=0$, and each hypothesis $\hat{h}_t\in \sett{\hat{h}_1,\ldots,\hat{h}_T}$ is representable as set of size $\O(\Lambda_{\RE})$. This defines a compression scheme of size $\Lambda_{\RE}T$, and $\bar{h}$ can be  reconstructed from a compression set of points from $S$ of size 
$\Lambda_{\RE}T$.

Recall that $S' \subseteq S$ is a maximal subset 
such that $\inf_{h\in\H}\widehat{\risk}\paren{h;S'}=0$ which implies that $\widehat{\risk}\paren{\bar{h};S}\leq \inf_{h\in\H}\widehat{\risk}\paren{h;S}$.
Plugging it into a data-dependent compression generalization bound (Lemma \ref{lem:sample-compression-bernstein}), we obtain a sample complexity of $\Lambda_{\AG}\paren{\epsilon,\delta,\H} = 
\O\paren{\ \frac{\VC(\H)}{\epsilon^2}\log^2\paren{\frac{\VC(\H)}{\epsilon^2}}+\frac{1}{\epsilon^2}\log\frac{1}{\delta}}$, as mentioned in \cref{thm:partial-sample-complexity-agnostic}.
\section{Supervised robust learning algorithms}\label{app:algo-agnostic-robust}
We overview the algorithms of \citet[proofs of Theorems 4 and 8]{montasser2019vc}. Their construction is based on sample compression methods explored in \cite{hanneke2019sample,moran2016sample}.

Let $\H \subseteq \sett{0,1}^\X$, fix a distribution $\D$ over the input space $\X\times \Y$. Let $S = \sett{(x_1,y_1),\dots,(x_m,y_m)}$ be an i.i.d. training sample from a robustly realizable distribution $\D$ by $\H$ , namely $\inf_{h\in\H}\Risk_{\U}\paren{h;\D}=0$.
Denote $d = \VC(\H)$, $d^* = \dualVC(\H)$ is the \textit{dual VC-dimension}. Fix $\epsilon,\delta \in (0,1)$.
\begin{enumerate}[leftmargin=0.5cm]
    \item Define the inflated training data set 
    $$S_\U = \bigcup_{i
    \in [n]}\sett{(z,y_{I(z)}):z\in \U(x_i)},$$ 
    where $I(z) =\min\sett{i\in [n]:z\in \U(x_i)}$. The goal is to construct a compression scheme the is consistent with $S_\U$.
    \item Discretize $S_{\U}$ to a finite set $\bar{S}_\U$.
        Define the class of hypotheses with zero robust error on every $d$ points in $S$,
    $$\hat{\H}=\sett{\RERM_{\H}(S'): S'\subseteq S, |S'|=d},$$
    where $\RERM_\H$ maps any labeled set to a hypothesis in $\H$ with zero robust loss on this set. 
    The cardinality of this class is bounded as following
    $$|\hat{\H}|= {n \choose d}\leq \left(\frac{en}{d}\right)^{d}.$$
    Discretize $S_\U$ to a finite set using the finite class $\hat{\H}$. 
    Define the \textit{dual class} $\H^* \subseteq \sett{0,1}^\H$ of $\H$ as the set of all functions $f_{(x,y)}: \H \rightarrow \sett{0,1}$ defined by $f_{(x,y)}(h) = \I\sqparen{h(x)\neq y}$, for any $h\in \H$ and $(x,y)\in S_{\U}$. If we think of a binary matrix where the rows consist of the distinct hypotheses and the columns are points, then the dual class corresponds to the transposed matrix where the distinct rows are points and the columns are hypotheses. A discretization $\bar{S}_\U$ will be defined by the dual class of $\hat{\H}$. Formally, $\bar{S}_\U \subseteq S_\U$ consists of exactly one $(x,y)\in S_\U$ for each distinct classification $\sett{f_{(x,y)}(h)}_{h\in \hat{\H}}$. In other words, $\hat{\H}$ induces a finite partition of $S_\U$ into regions where every $\hat{h}\in \hat{\H}$ suffers a constant loss $\I\sqparen{\hat{h}(x)\neq y}$ in each region, and the discretization $\bar{S}_\U$ takes one point per region.
    By Sauer's lemma \citep{vapnik2015uniform,sauer1972density},  for $n > 2d$,
    $$|\bar{S}_\U| 
    \leq 
    \left( \frac{e|\hat{\H}|}{d^*}\right)^{d^*}
    \leq 
    \left( \frac{e^2n}{dd^*}\right)^{dd^*},$$
   
   \item
    Execute the following modified version of the algorithm $\alpha$-boost \citep[pages 162-163]{schapire2013boosting} on the discretized set $\bar{S}_\U$, with parameters $\alpha = \frac{1}{3}$ and number of boosting rounds $T=\O\paren{\log\paren{|\bar{S}_\U|}}=\O\left( dd^*\log (n)\right)$.
    
    \begin{algorithm}[H]
    \caption{Modified $\alpha$-boost}\label{alg:alpha-boost}
    \textbf{Input:} $\H, S, \bar{S}_\U, d, \RERM_\H$.\\
    \textbf{Parameters:} $\alpha, T$.\\
    \textbf{Initialize} $P_1$ = Uniform($\bar{S}_\U$).\\
     For $t=1,\ldots,T$:
        \begin{enumerate}
            \item Find $\O(d)$ points ${S_t}\subseteq \bar{S}_\U$ such that every $h\in \H$ with $\widehat{\risk}(h;S_t)=0$ has $\risk(h;P_t)\leq 1/3$.
            \item Let $S'_t$ be the original $\O(d)$ points in $S$ with $S_t\subseteq \bigcup_{(x,y)\in {S'_t}}\bigcup \sett{(z,y): z\in \U(x)}$.
            \item Let $\hat{h}_t = \RERM_\H(S'_t)$.
            \item For each $(x,y)\in \bar{S}_\U$:
            \begin{itemize}
                \item[] $P_{t+1}(x,y) \propto P_{t}(x,y)e^{-\alpha \I\sett{\hat{h}_t(x)=y}}$
            \end{itemize}
        \end{enumerate}
    \textbf{Output:} classifiers $\hat{h}_1,\ldots,\hat{h}_T$ and sets $S'_1,\ldots,S'_T$.
\end{algorithm}

    \item Output the majority vote
    $\bar{h}=\maj\paren{\hat{h}_1,\ldots,\hat{h}_T}$.
\end{enumerate}
We are guaranteed that  $\widehat{\risk}_{\U}\paren{\bar{h};S}=0$,
and each hypothesis $\hat{h}_t\in \sett{\hat{h}_1,\ldots,\hat{h}_T}$ is representable as set $S'_t$ of size $\O(d)$. This defines a compression function $\kappa(S)= \bigcup_{t\in [T]} S'_t$.
Thus, $\bar{h}$ can be  reconstructed from a compression set of size 
$$dT = \O\left(d^2d^*\log(n)\right).$$

This compression size can be further reduced to $\O\left(dd^*\right)$, using a sparsification technique introduced by \citet{moran2016sample,hanneke2019sample}, by randomly choosing $\O(d^*)$ hypotheses from $\sett{\hat{h}_1,\ldots,\hat{h}_T}$. The proof follows via standard uniform convergence argument. Plugging it into a compression generalization bound, we have a sample complexity of $\Tilde{\O}\paren{\frac{dd^*}{\epsilon}+\frac{\log\frac{1}{\delta}}{\epsilon}}$, in the realizable robust  case.

\paragraph{Agnostic case.}
The construction follows a reduction to the realizable case suggested by \citet{david2016supervised}. 
Denote $\Lambda_{\RE} = \Lambda_{\RE}(1/3,1/3,\H,\U)$, the sample complexity of $(1/3,1/3)$-$\pac$ learn $\H$ with respect to a perturbation function $\U$, in the realizable robust  case.

Using a robust $\ERM$, find the maximal subset $S'\subseteq S$
such that $\inf_{h\in\H}\widehat{\risk}_{\U}\paren{h;S'}=0$. Now, $\Lambda_{\RE}$ samples suffice for weak robust learning for any distribution $\D$ on $S'$.

Execute the $\alpha$-boost algorithm \citep[pages 162-163]{schapire2013boosting} on $S'$ for the robust loss function, with parameters $\alpha = \frac{1}{3}$ and number of boosting rounds $T=\O\paren{\log\paren{|S'|}}$, where each weak learner is trained on $\Lambda_{\RE}$ samples. The returned hypothesis
$\bar{h}=\maj\paren{\hat{h}_1,\ldots,\hat{h}_T}$
satisfies that  $\widehat{\risk}_{\U}\paren{\bar{h};S'}=0$, and each hypothesis $\hat{h}_t\in \sett{\hat{h}_1,\ldots,\hat{h}_T}$ is representable as set of size $\O(\Lambda_{\RE})$. This defines a compression scheme of size $\Lambda_{\RE}T$, and $\bar{h}$ can be  reconstructed from a compression set of points from $S$ of size 
$\Lambda_{\RE}T$.

Recall that $S' \subseteq S$ is a maximal subset 
such that $\inf_{h\in\H}\widehat{\risk}_{\U}\paren{h;S'}=0$ which implies that $\widehat{\risk}_{\U}\paren{\bar{h};S}\leq \inf_{h\in\H}\widehat{\risk}_{\U}\paren{h;S}$.
Plugging it into a compression generalization bound (Lemma \ref{lem:compression-generalization} holds for the robust loss function as well), we have a sample complexity of $\Tilde{\O}\paren{\frac{\Lambda_{\RE}}{\epsilon^2}+\frac{\log\frac{1}{\delta}}{\epsilon^2}}$, which translates into $\Tilde{\O}\paren{\frac{dd^*}{\epsilon^2}+\frac{\log\frac{1}{\delta}}{\epsilon^2}}$, in the agnostic robust case.

\end{document}